\documentclass{article}

\usepackage{microtype}
\usepackage{graphicx}
\usepackage{subfigure}
\usepackage{booktabs} 
\usepackage{wrapfig}
\usepackage{hyperref}

\usepackage[accepted]{icml2025}

\usepackage{amsmath}
\usepackage{amssymb}
\usepackage{mathtools}
\usepackage{amsthm}

\usepackage[capitalize,noabbrev]{cleveref}

\theoremstyle{plain}
\newtheorem{theorem}{Theorem}[section]

\newtheorem{lemma}[theorem]{Lemma}

\theoremstyle{definition}
\newtheorem{definition}[theorem]{Definition}

\theoremstyle{remark}

\usepackage[textsize=tiny]{todonotes}

\icmltitlerunning{Robustifying Fourier Features Embeddings for Implicit Neural Representations}

\begin{document}
\twocolumn[
\icmltitle{Robustifying Fourier Features Embeddings for Implicit Neural Representations}

\icmlsetsymbol{equal}{*}

\begin{icmlauthorlist}
\icmlauthor{Mingze Ma}{UT}
\icmlauthor{Qingtian Zhu}{UT}
\icmlauthor{Yifan Zhan}{UT}
\icmlauthor{Zhengwei Yin}{UT}
\icmlauthor{Hongjun Wang}{UT}
\icmlauthor{Yinqiang Zheng}{UT}
\end{icmlauthorlist}

\icmlaffiliation{UT}{Department of Information Science and Technology, The University of Tokyo, Tokyo, Japan}

\icmlcorrespondingauthor{Mingze Ma}{ma-mingze156@g.ecc.u-tokyo.ac.jp}

\icmlkeywords{Machine Learning, ICML}

\vskip 0.3in
]



\printAffiliationsAndNotice{}
\begin{abstract}
Implicit Neural Representations (INRs) employ neural networks to represent continuous functions by mapping coordinates to the corresponding values of the target function, with applications e.g., inverse graphics. 
However, INRs face a challenge known as spectral bias when dealing with scenes containing varying frequencies.
To overcome spectral bias, the most common approach is the Fourier features-based methods such as positional encoding.
However, Fourier features-based methods will introduce noise to output, which degrades their performances when applied to downstream tasks.
In response, this paper initially hypothesizes that combining multi-layer perceptrons (MLPs) with Fourier feature embeddings mutually enhances their strengths, yet simultaneously introduces limitations inherent in Fourier feature embeddings.
By presenting a simple theorem, we validate our hypothesis, which serves as a foundation for the design of our solution.
Leveraging these insights, we propose the use of multi-layer perceptrons (MLPs) without additive terms (referred to as bias-free MLPs) as adaptive linear filters. These bias-free MLPs locally suppress unnecessary frequencies while enriching embedding frequencies, which theoretically reduces the lower bound of the loss of the MLPs.
Additionally, we propose a line-search-based algorithm to adjust the filter's learning rate dynamically, achieving a balance between the adaptive linear filter module and the INRs which further promote the performance.
Extensive experiments demonstrate that our proposed method consistently improves the performance of INRs on typical tasks, including image regression, 3D shape regression, and inverse graphics.
\end{abstract}



\section{Introduction}
Implicit Neural Representations (INRs), which fit the target function using only input coordinates, have recently gained significant attention.
By leveraging the powerful fitting capability of Multilayer Perceptrons (MLPs), INRs can implicitly represent the target function without requiring their analytical expressions. 
The versatility of MLPs allows INRs to be applied in various fields, including inverse graphics~\citep{mildenhall2021nerf, barron2023zip, martin2021nerf}, image super-resolution~\citep{chen2021learning, yuan2022sobolev, gao2023implicit}, 
image generation~\citep{skorokhodov2021adversarial}, and more~\citep{chen2021nerv, strumpler2022implicit, shue20233d}.
\begin{figure}
    \includegraphics[width=0.5\textwidth]{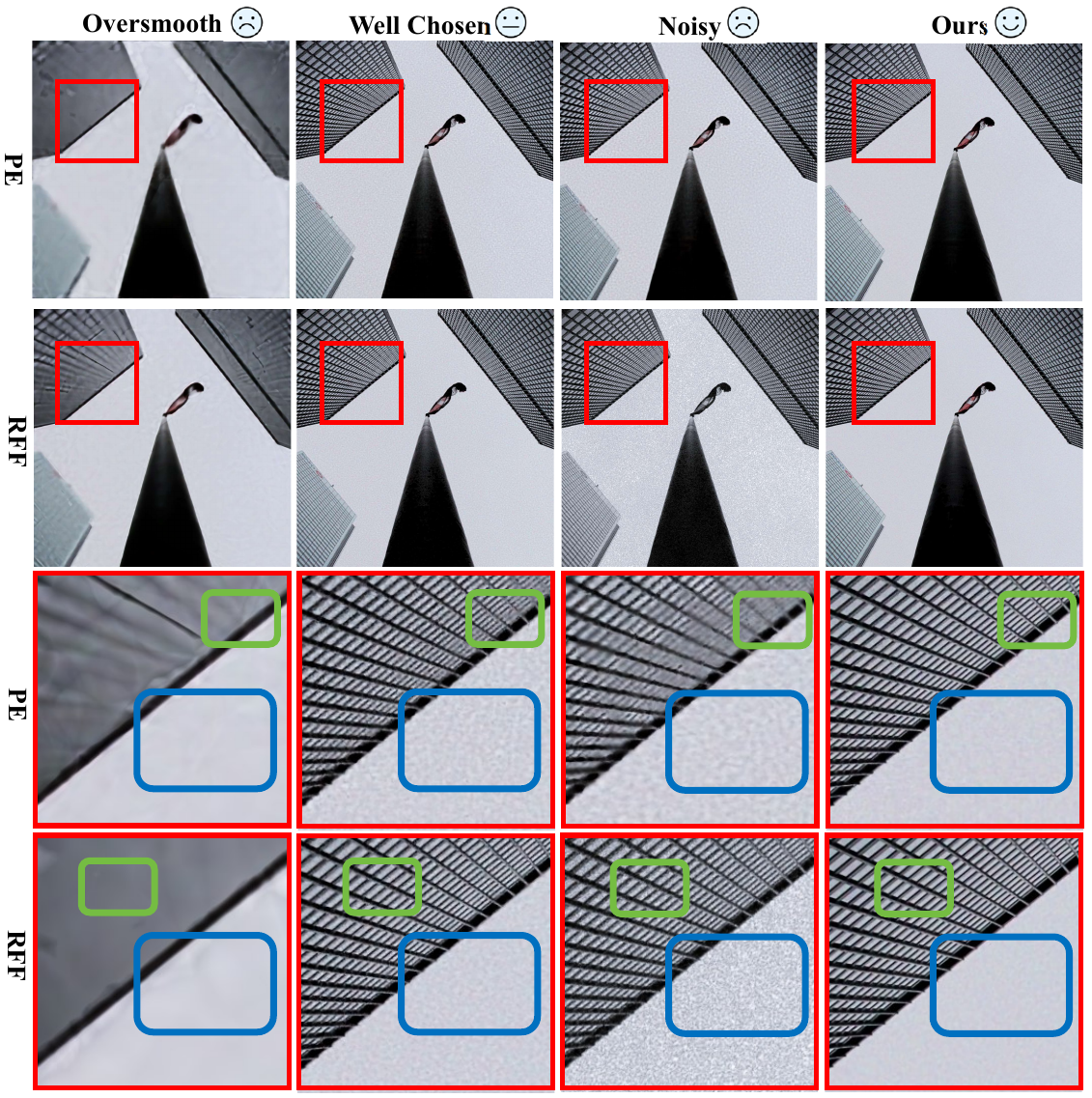}
    \caption{As illustrated at the circled blue regions and green regions, it can be observed that even with well-chosen standard deviation/scale, as experimented in \autoref{figure:combined}, the results are still unsatisfactory. However, using our proposed method, the noise is significantly alleviated while further enhancing the high-frequency details.}
    \label{fig:var}
    \vspace{-10pt}
\end{figure}

\begin{figure*}[!ht]
    \centering
    \begin{minipage}[b]{0.25\textwidth}
        \centering
        \includegraphics[width=1.\textwidth]{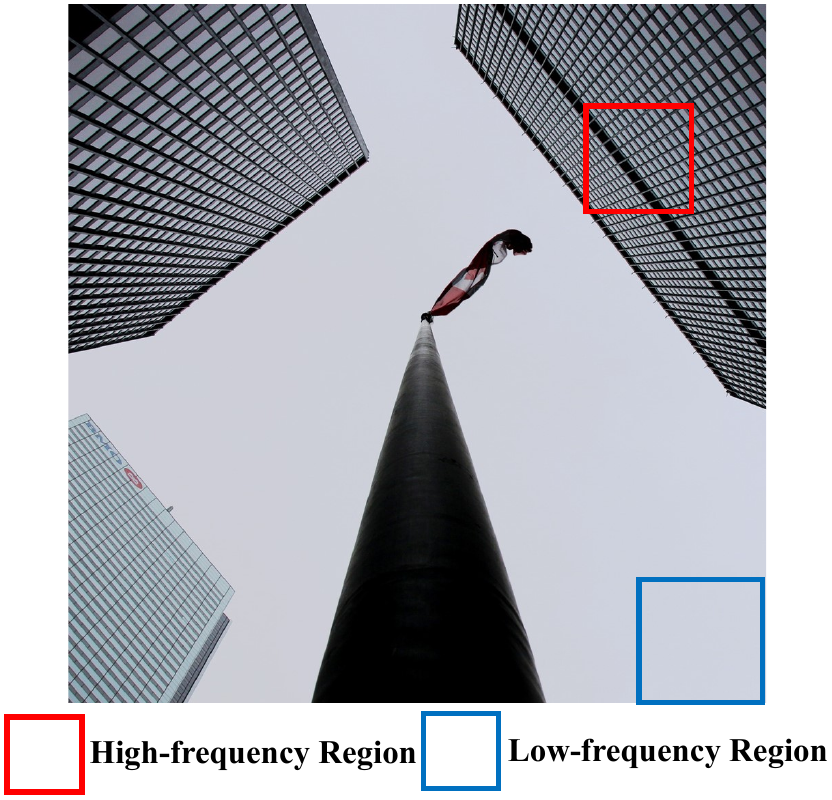} 
        \label{figure:small_image}
        \vspace{-20pt}
    \end{minipage}%
    \hfill
    \begin{minipage}[b]{0.75\textwidth}
        \centering
        \includegraphics[width=1.\textwidth]{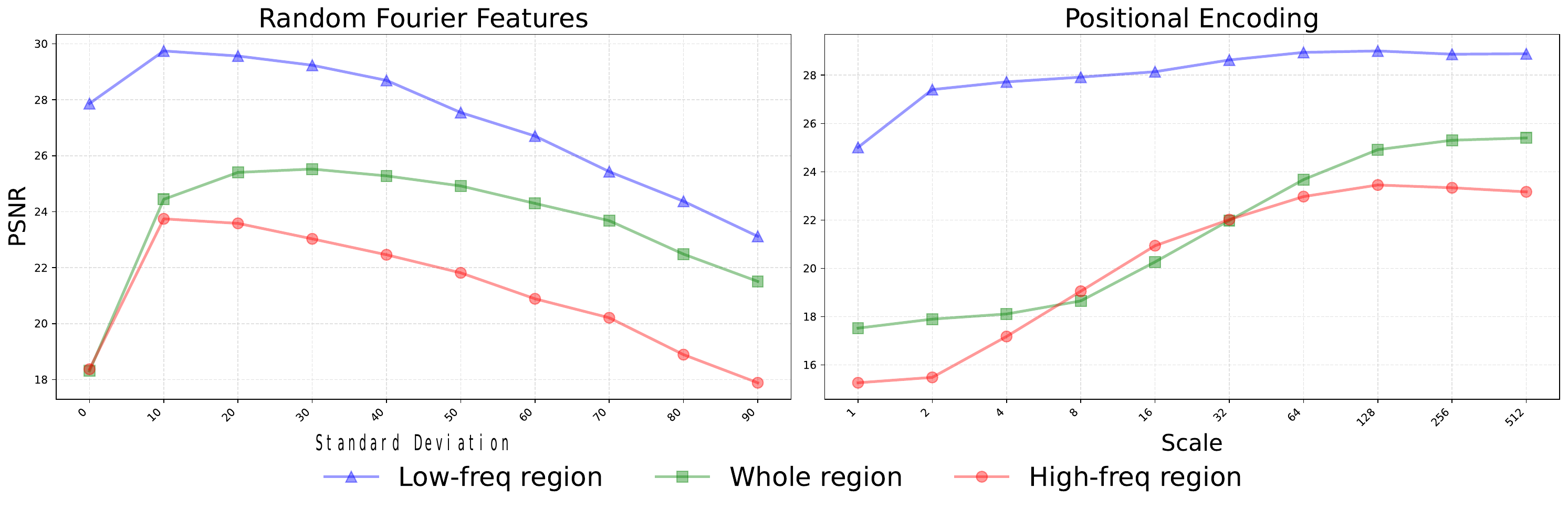} 
        \vspace{-20pt}
        \label{figure:large_image}
        
    \end{minipage}
    \caption{We test the performance of MLPs with Random Fourier Features (RFF) and MLPs with Positional Encoding (PE) on a 1024-resolution image to better distinguish between high- and low-frequency regions, as demonstrated on the left-hand side of this figure. We find that the performance of MLPs+RFF degrades rapidly with increasing standard deviation compared with MLPs+PE. Since positional encoding is deterministic, scale=512 can be considered to have standard deviation around 121.}
    \label{figure:combined}
    \vspace{-10pt}
\end{figure*}
Varying the sampling standard deviation/scale may lead to degradation results, as shown in \autoref{figure:combined}.
However, MLPs face a significant challenge known as the spectral bias, where low-frequency signals are typically favored during training~\citep{rahaman2019spectral}. 
A common solution is to map coordinates into the frequency domain using Fourier features, such as Random Fourier Features and Positional Encoding, which can be understood as manually set high-frequency correspondence prior to accelerating the learning of high-frequency targets.~\citep{tancik2020fourier}. 
This embeddings widely applied to the INRs for novel view synthesis~\citep{mildenhall2021nerf,barron2021mip}, dynamic scene reconstruction~\citep{pumarola2021d}, object tracking~\citep{wang2023tracking}, and medical imaging~\citep{corona2022mednerf}.

Although many INRs' downstream application scenarios use this encoding type, it has certain limitations when applied to specific tasks.
It depends heavily on two key hyperparameters: the sampling standard deviation/scale (available sampling range of frequencies) and the number of samples.
Even with a proper choice of sampling standard deviation/scale, the output remains unsatisfactory, as shown in \autoref{fig:var}: Noisy low-frequency regions and degraded high-frequency regions persist with well chosen sampling standard deviation/scale with the grid-searched standard deviation/scale, which may potentially affect the performance of the downstream applications resulting in noisy or coarse output.
However, limited research has contributed to explaining the reason and finding a proper frequency embeddings for input~\citep{landgraf2022pins, yuce2022structured}.

In this paper, we aim to offer a potential explanation for the high-frequency noise and propose an effective solution to the inherent drawbacks of Fourier feature embeddings for INRs.
Firstly, we hypothesize that the noisy output arises from the interaction between Fourier feature embeddings and multi-layer perceptrons (MLPs). We argue that these two elements can enhance each other's representation capabilities when combined. However, this combination also introduces the inherent properties of the Fourier series into the MLPs.
To support our hypothesis, we propose a simple theorem stating that the unsampled frequency components of the embeddings establish a lower bound on the expected performance. This underpins our hypothesis, as the primary fitting error in finitely sampled Fourier series originates from these unsampled frequencies.

Inspired by the analysis of noisy output and the properties of Fourier series expansion, we propose an approach to address this issue by enabling INRs to adaptively filter out unnecessary high-frequency components in low-frequency regions while enriching the input frequencies of the embeddings if possible.
To achieve this, we employ bias-free (additive term-free) MLPs. These MLPs function as adaptive linear filters due to their strictly linear and scale-invariant properties~\citep{mohan2019robust}, which preserves the input pattern through each activation layer and potentially enhances the expressive capability of the embeddings.
Moreover, by viewing the learning rate of the proposed filter and INRs as a dynamically balancing problem, we introduce a custom line-search algorithm to adjust the learning rate during training. This algorithm tackles an optimization problem to approximate a global minimum solution. Integrating these approaches leads to significant performance improvements in both low-frequency and high-frequency regions, as demonstrated in the comparison shown in \autoref{fig:var}.
Finally, to evaluate the performance of the proposed method, we test it on various INRs tasks and compare it with state-of-the-art models, including BACON~\citep{lindell2022bacon}, SIREN~\citep{sitzmann2020implicit}, GAUSS~\citep{ramasinghe2022beyond} and WIRE~\citep{saragadam2023wire}. 
The experimental results prove that our approach enables MLPs to capture finer details via Fourier Features while effectively reducing high-frequency noise without causing oversmoothness.
To summarize, the following are the main contributions of this work:
\begin{itemize}
    \item From the perspective of Fourier features embeddings and MLPs, we hypothesize that the representation capacity of their combination is also the combination of their strengths and limitations. A simple lemma offers partial validation of this hypothesis.

    \item  We propose a method that employs a bias-free MLP as an adaptive linear filter to suppress unnecessary high frequencies. Additionally, a custom line-search algorithm is introduced to dynamically optimize the learning rate, achieving a balance between the filter and INRs modules.

    \item To validate our approach, we conduct extensive experiments across a variety of tasks, including image regression, 3D shape regression, and inverse graphics. These experiments demonstrate the effectiveness of our method in significantly reducing noisy outputs while avoiding the common issue of excessive smoothing.
\end{itemize}

\section{Related works}
Implicit Neural Representations are designed to learn continuous representations of target functions by taking advantages of the approximation power of neural networks.
Their inherent continuous property can beneficial in many cases like video compression~\citep{chen2021nerv,strumpler2022implicit}, 3D modeling~\citep{park2019deepsdf,atzmon2020sal,9010266,gropp2020implicit,sitzmann2019scene} and volume rendering~\citep{pumarola2021d, barron2021mip,martin2021nerf,barron2023zip}.
However, simply employing MLPs may result in spectral bias, where oversmoothed outputs are generated due to the inherent tendency of MLPs to prioritize learning low-frequency components first. Consequently, many studies have focused on these drawbacks and explored various methods to address this issue.
The most straightforward way to address this issue is by projecting the coordinates into the higher dimension~\citep{tancik2020fourier, wang2021spline}.
However, these methods can lead to noisy outputs if there is a mismatch in the embeddings variance.
To address this, \citet{landgraf2022pins} propose dividing the Random Fourier Features into multiple levels of detail, allowing the MLPs to disregard unnecessary high-frequency components. Another type of approach to mitigating the spectral bias introduced by the ReLU activation function, as proposed by \citet{sitzmann2020implicit}, \citet{ramasinghe2022beyond}, \citet{saragadam2023wire}, and \citet{shenouda2024relus}, is to modify the activation function itself by using alternatives such as the Sine function, Wavelets, or a combination of ReLU with other functions. There are also efforts to modify network structures to mitigate spectral bias~\citep{mujkanovic2024neural}. 
\citet{lindell2022bacon} introduce a network design that treats MLPs as filters applied to the input of the next layer, known as Multiplicative Filter Networks (MFNs). 
Additionally, based on the discrete nature of signals like images and videos, grid-based approaches (e.g., Grid Tangent Kernel~\citep{zhao2024grounding}, DINER~\citep{xie2023diner}, and Fourier Filter Bank~\citep{wu2023neural}) have been proposed to address spectral bias, as the grid property allows for sharp changes in features, which facilitates learning fine details.
Even though, there are some prior works trying to solve the inherent problems of Fourier features embeddings ~\citep{landgraf2022pins, yuce2022structured, hertz2021sape, saratchandran2024sampling}, limited research has addressed both the underlying causes of high-frequency noise and provides a non-heuristic solution even if these embeddings are widely employed into many downstream tasks.
\section{Preliminary}
\begin{figure*}[!ht]
    \centering
    \includegraphics[width=0.75\textwidth]{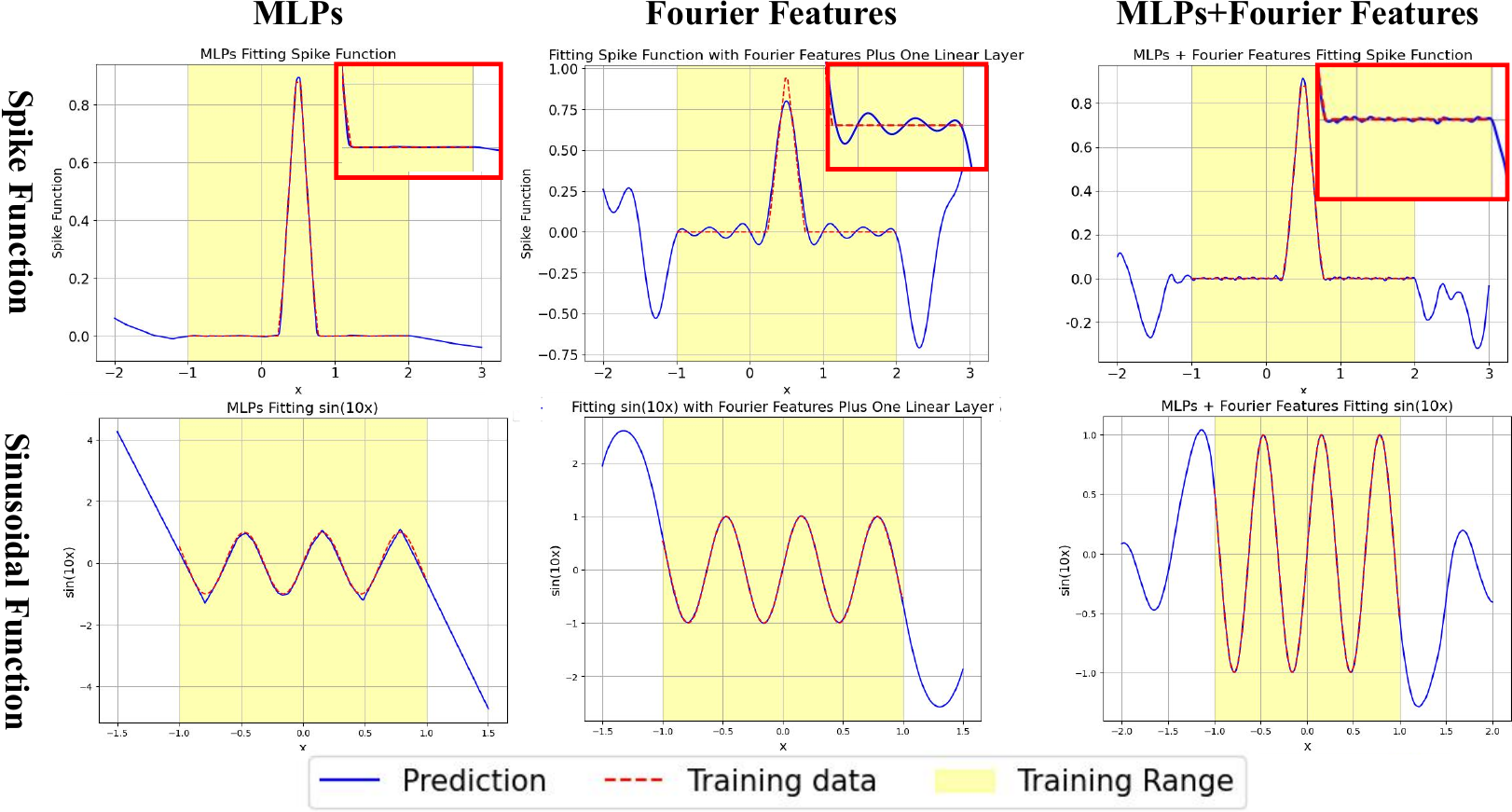}
    \caption{We demonstrate our hypothesis by using three models (MLPs, Fourier features with one linear layer and their combination (i.e. MLPs with Fourier features embeddings)) to fit two kinds of functions. The result demonstrate that combining MLPs with Fourier Features can actually combine their representation capability. These highlighted red boxes demonstrate that MLPs with Fourier features also involve the representation capability of the Fourier features where there are high-frequency fluctuations in the flat regions due to the non-differentiable point in spike function.}
    \label{fig:fourier_series} 
\end{figure*}
\subsection{Fourier Features Embeddings}
\label{sec:FF}
Fourier features embeddings are common embedding methods to alleviate spectral bias. As a type of embedding that maps inputs into the frequency domain, they can be expressed by the function $\gamma(\mathbf{v}): \mathbb{R}^d\rightarrow\mathbb{R}^N$, where \(d\) is the input coordinate dimension and \(N\) is the embedding dimension. The two most common types are Random Fourier Features (RFF) and Positional Encoding (PE), which can both be represented by a single formula with slight variations in their implementation.
\begin{definition}
[$\mathbf{Fourier\, features}$] $\mathbf{Fourier\, features}$ can be generally defined as a function such that $\gamma(\mathbf{v}): \mathbb{R}^d\rightarrow\mathbb{R}^N$
\begin{equation}
\begin{split}
\gamma(\mathbf{v}) = &[sin(2\pi\mathbf{b}_i^{\top}\mathbf{v}), cos(2\pi \mathbf{b}_i^{\top}\mathbf{v})]_{i\in[N]}\\
&[N] = \{1, 2, 3,\cdots, N\},\,\mathbf{b}_i \in\mathbb{R}^{d\times 1}\\
\end{split}
\end{equation}
\,\,\,\,\,\,\,$\mathbf{Positional\, Encoding}$: $\mathbf{b}_i=\mathbf{s}^{\frac{i}{N}}, \, for\, i\in[N]$. It applies log-linearly spaced frequencies for each dimension, with the scale $\mathbf{s}$ and size of embedding \(N\) as hyperparameters, and includes only on-axis frequencies.

\,\,\,\,\,\,\,$\mathbf{Random\, Fourier\, Features}$: $\mathbf{b}_i\sim\mathcal{N}(0, \Sigma)$. Typically, this is an isotropic Gaussian distribution, meaning that $\Sigma$ has only diagonal entries. Other distributions, such as the Uniform distribution, can also be used, though the Gaussian distribution remains the most common choice.
\label{def:ff}
\end{definition}
\subsection{Bias-free ReLU MLPs}
\label{section:BF}
For tasks with identical input and output dimensions, additive term-free ReLU MLPs (also referred to as bias-free MLPs) can, in contrast to standard MLPs, be regarded as locally strictly linear operators (Lemma~\ref{lemma:BF1}) with a scale-invariant property (Lemma~\ref{lemma:BF2}), as demonstrated in the following lemmas.

\begin{lemma}
    (\citet{mohan2019robust}) For a Bias-free ReLU activation function ($\sigma(\cdot)$) MLPs $f_{BF}(\cdot): \mathbb{R}^N\rightarrow\mathbb{R}^N$ with L layers, matrix at each layer is denoted as $W^l$ for $l=1,\cdots,L$. Then, the MLPs can be written as $f_{BF}(\mathbf{x})=\mathbf{A}_\mathbf{x}\mathbf{x}$.
\label{lemma:BF1}
\end{lemma}

\begin{lemma}
(\citet{mohan2019robust}) For a Bias-free ReLU activation function ($\sigma(\cdot)$) MLPs $f_{BF}(\cdot): \mathbb{R}^N\rightarrow\mathbb{R}^N$ with L layers, matrix at each layer is denoted as $W^l$ for $l=1,\cdots,L$.  For any input $\mathbf{x}$ and any nonnegative constant $\alpha$,
\begin{equation}
f_{BF}(\alpha \mathbf{x}) = \alpha f_{BF}(\mathbf{x}).
\end{equation}
\label{lemma:BF2}
\vspace{-15pt}
\end{lemma}


To connect these properties with our method, intuitively, scaling the input by a constant should preserve the input's frequency. Since the goal for MLPs is to filter based on frequency pattern rather than amplitude of a signal for a singe coordinate, the scale-invariant property ensures that scaling does not affect the result. By incorporating an additive term in the MLPs, the network function is modified to \( f_{BF}(\mathbf{x}) = \mathbf{A}_\mathbf{x} \mathbf{x} + \mathbf{b}_\mathbf{x} \), can be shown in a similar proof to Lemma \ref{lemma:BF1}. However, this adjustment disrupts the scale-invariant property and alters the activation pattern, which is undesirable.

\section{Analysis and Motivations}
\begin{figure*}[!ht]
    \centering\includegraphics[page=1, width=1\textwidth]{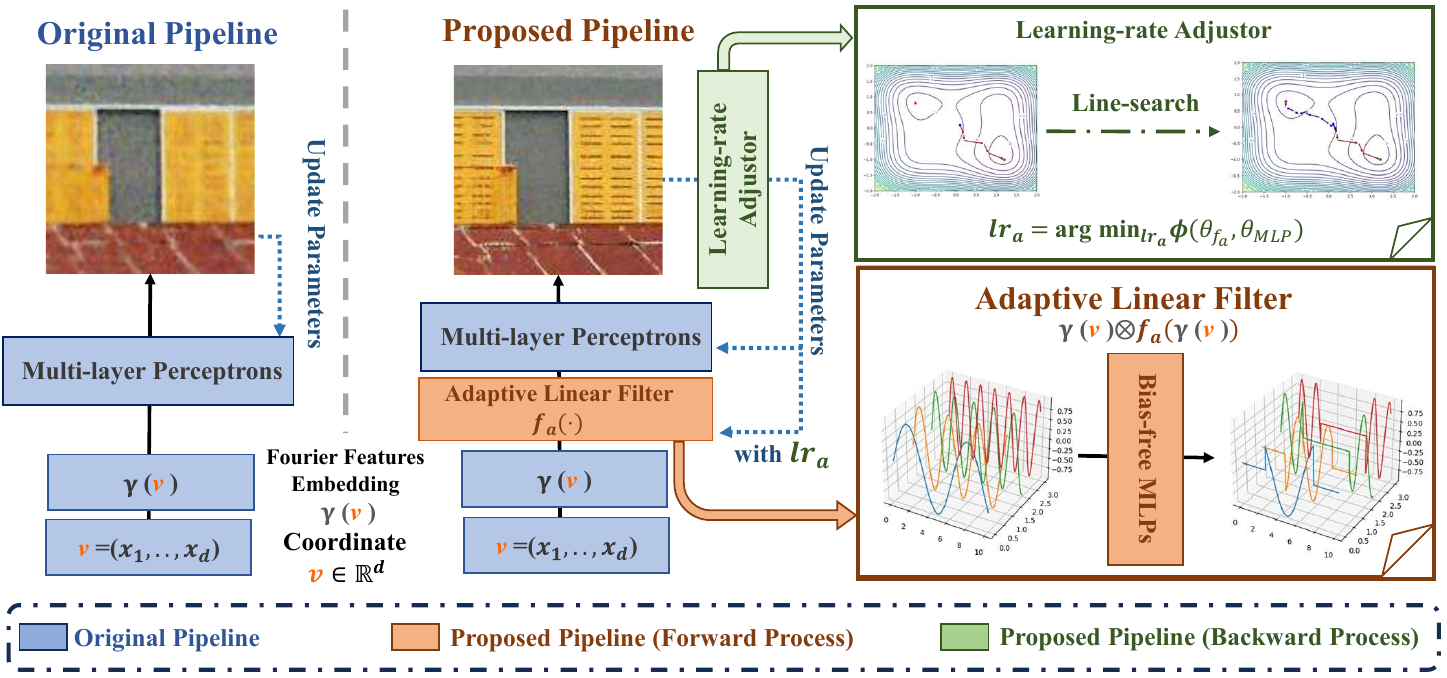}
    \caption{The pipeline of our method introduces two additional modules compared to the original approach. The first module, an adaptive linear filter, removes unnecessary frequency components at the pixel level, reducing high-frequency noise during regression. The second module dynamically adjusts the learning rate during training to optimize the approximated loss for the next step, achieving dynamical balance. Together, these modules result in cleaner and more detailed images.}
    \label{fig:T}
   \vspace{-10pt} 
\end{figure*}
It is widely known that the ReLU-activated function can form a spike function as shown in \autoref{fig:fourier_series} using two neurons in the hidden layer. This spike function forms a Shauder basis for the $\mathcal{C}[0,1]$ space, which means that that every continuous function can be uniquely decompose into a linear combination of countable infinitely many spike functions. In contrast, using Fourier features embeddings with a single linear layer can be recognized as the Fourier series to fit the target function which can also be considered as a Shauder basis in [0, 1] (since the function can be periodic with T=1).

Therefore, using three layers of MLPs (one input layer, one hidden layer, and one output layer) requires infinitely many neurons in hidden layers to fit a sinusoidal function. Since if we consider each two neurons as a spike function times a constant and fitting a sinusoidal function using spike function requires countable infinity of those functions, this is equivalent to using infinitely many neurons to perfectly fit a sinusoidal function. On the other hand, using the sinusoidal function to fit the spike function also requires infinitely many sinusoidal functions. This is because the Fourier transform of the spike function has infinite support, and we need infinitely many delta functions (the Fourier transform of sinusoidal functions) to fit the Fourier transform of the spike function.

By the above simple analysis, to understand the behaviour of the noisy output, we propose a hypothesis that combining both MLPs and Fourier features can be considered as combining their representation capacity as shown \autoref{fig:fourier_series} (also the convergence plot in \autoref{sec:spike}). This provides an explanation for the noisy output that discontinuous or non-differentiable property (usually from change of objects or change of color) of the image will require infinity many sinusoidal functions to suppress the high-frequencies components on the continuous regions. This can also be partially proved by considering the Neural Tangent Kernel theory. If MLPs plus Fourier features embeddings can be interpreted as the linear combination of sinusoidal functions (proved by the Lemma. \ref{lemma1}), then by simple deduction based on triangular inequality and Orthogonal Decomposition Theorem, the unselected frequencies by the Fourier features embeddings can form an lower bound for the theoretical performance (proved by Lemma. \ref{lemma:1}).
To avoid noises, one idea is to expand the frequency band and suppress the high-frequencies components on the flat regions which leads to our adaptive linear filter.

\section{Methods}

In this section, we present our solution grounded in the previous analysis. The proposed method has two main components: (i) an adaptive linear filter that automatically adjust the input embeddings which also potentially introduces more frequencies components, and (ii) a learning-rate adjustor that uses the line-search method during back-propagation to dynamically adjust the filter's learning rate. The full pipeline is illustrated in \autoref{fig:T}.
\subsection{Bias-Free MLPs as Adaptive Linear Filter}
In section \ref{section:BF}, we have introduced the properties of bias-free MLPs. In this section, we will further explore the advantages of using bias-free MLPs as adaptive linear filters. Let $\gamma(\mathbf{v})$ represent the Fourier feature embeddings of the input $\mathbf{v}$, and let \( f_{a}(\cdot) \) denote the adaptive linear filter, implemented as a bias-free ReLU-activated MLP. Our adaptive linear filter, applied to the input embeddings $\gamma(\mathbf{v})$, can be defined as $f_{a}(\gamma(\mathbf{v}))\otimes\gamma(\mathbf{v})$, where $\otimes$ denotes the Hadamard product. In other words, we apply \( f_{a}(\cdot) \) to the embedding to obtain a linear filter, which is then applied to the original embeddings to modify them. The filtered result, \( f_{a}(\gamma(\mathbf{v})) \otimes \gamma(\mathbf{v}) \), will subsequently serve as the final filtered embeddings to the INRs.

Moreover, these bias-free MLPs not only preserve the input pattern after the activation function but also extend the range of embeddings' frequencies, which could theoretically lower the performance bound discussed in  Lemma. \ref{lemma:1}, in contrast to using a simple mask as the linear filter. Note that \( f_{a}(\cdot) \) is bias-free and can be expressed as \( f_{a}(\mathbf{x}) = \mathbf{A}_{\mathbf{x}} \mathbf{x} \), as established in Lemma \ref{lemma:BF1}. Substituting this into \( f_{a}(\gamma(\mathbf{v})) \otimes \gamma(\mathbf{v}) \) yields \( \mathbf{A}_{\gamma(\mathbf{v})} \gamma(\mathbf{v}) \otimes \gamma(\mathbf{v}) \). 
\begin{figure}[!b]
 \vspace{-5pt}
    \centering
    \includegraphics[width=0.5\textwidth]{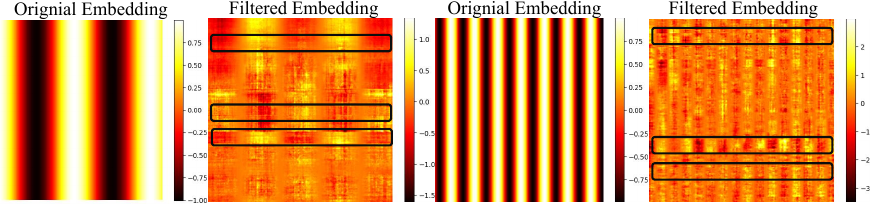}
    \vspace{-10pt}
    \caption{The black boxes highlight frequency bands with varying phases or frequencies, supporting our derivation that the adaptive linear filter can enrich the embeddings.
}
    \label{figure:fb}
   
\end{figure}
For simplicity, we consider the one-dimensional case, which can easily be extended to higher dimensions by following a similar deduction. Let \( a_{ij} \) represent an entry of \( \mathbf{A}_{\gamma(\mathbf{v})} \), and \( \sin(b_i v) \) denote the \( i^{\text{th}} \) entry of \( \gamma(v) \) without loss of generality. Since the Hadamard product performs entry-wise multiplication, the output for a channel with frequency \( b_i \) is given by \( o_i = \sum_j a_{ij} \sin(b_j v) \sin(b_i v) \). Using the trigonometric identity \( \sin(b_j v) \sin(b_i v) = \frac{1}{2} (\sin((b_j + b_i) v) + \sin((b_j - b_i) v)) \), \( f_a(\cdot) \) can adjust input frequencies for the INRs through \( a_{ij} \), enabling enriched frequency choices. It can also preserve the original embeddings by setting \( a_{ij} = \frac{1}{2N \sin(b_j v)} \), where \( N \) is the embedding length, or block specific frequencies by setting \( a_{ij} = 0 \) for those entries. This simple analysis supports the implementation of the adaptive linear filter. Notably, the above derivation remains valid for MLPs with bias terms; however, they may introduce more complex, non-sinusoidal patterns (Comparison of the performance will be demonstrated in Section \ref{section:BFExp}).
 The illustration (\autoref{figure:fb}) from our experiments with mixing frequencies in a single frequency channel verifying our deduction.

\subsection{Line-searched based optimization}

During experiments, we observed that varying initial learning rates for the adaptive linear filter and INRs resulted in different performance outcomes. Balancing their learning rates is crucial: if INRs learn too quickly, the system risks local minima, hindering the filter's performance. Conversely, a high learning rate for the filter can cause excessive input fluctuations, preventing INRs from converging. Inspired by \citet{hao2021adaptive}, we aim to optimize the learning rate of the adaptive linear filter. By optimizing the loss function $f(\theta_{fa}, \theta_{MLP})$ as $\phi(lr_a) = f(\theta_{fa},  \theta_{MLP})$ during training (where $\theta_{fa}$ represents the parameters of the adaptive linear filter, $\theta_{MLP}$ represents the parameters of the INRs, $lr_a$ and $lr_I$ represent the learning rates for the adaptive filter and INR, respectively), we calculate the learning rate $lr_a$ for the adaptive linear filter at each iteration.
By applying the Taylor expansion of the loss function, this optimization problem can be approximated as a linear optimization problem, which can be evaluated efficiently without extensive computation. The complete algorithm and derivation are provided in the supplementary material.

\section{Experiments}
To validate the proposed method, we test it across various tasks, including image regression, 3D shape regression, and inverse graphics. All experiments are performed on a single RTX 4090 GPU, using an adaptive linear filter with 3 layers, each with the same width as the number of channels in the Fourier features embedding.
\subsection{Image Regression on Kodak Dataset}
In this section, we evaluate the performance of our proposed method on the high resolution (512$\times$768 or 768$\times$512) Kodak Dataset~\citep{article}. All baselines are trained using the mean squared error (MSE) loss function.

We benchmark our approach against several state-of-the-art baselines, including Multi-Layer Perceptrons (MLP) with Positional Encoding, MLP with Random Fourier Features, SIREN~\citep{sitzmann2020implicit}, GAUSS~\citep{ramasinghe2022beyond}, and WIRE~\citep{saragadam2023wire}. Each model is trained for 20,000 iterations to ensure convergence and taken the highest performance as the final result, with all hyperparameters including learning rate, layers, $\omega$ and s for other activation functions, aligned with the official implementations of the baseline methods in WIRE~\footnote{\url{https://github.com/vishwa91/wire.git}}. 

For the custom line-search algorithm employed in our method, we configure the maximum learning rate to $1 \times 10^{-3}$, with a minimum threshold of 0. To ensure a comprehensive comparison, we assess performance across three standard metrics: Peak Signal-to-Noise Ratio (PSNR), Structural Similarity Index (SSIM), and Learned Perceptual Image Patch Similarity (LPIPS)~\citep{zhang2018unreasonable}.

\begin{table}[!ht]
\vspace{-10pt}
    \centering
    \caption{MLP+PE+Ours achieves the best performance across all metrics, demonstrating superior reconstruction quality and visual fidelity. WIRE ranks second, excelling in SSIM and LPIPS. RFF-based methods perform poorly, likely due to their inherent limitation in fitting non-square images, as diagonal frequency components near the edges are harder to cover. Overall, our proposed method shows significant effectiveness, achieving approximately 10 PSNR improvement.}
     \vspace{9pt}
    \begin{tabular}{lccc}
        \toprule
        Methods & PSNR↑ & SSIM↑ & LPIPS↓ \\
        \midrule
        MLP+PE        & 25.67 &0.7001 &  0.2674\\
        MLP+RFF       & 26.58 & 0.7180 & 0.2307 \\
        
        SIREN         & 30.90 & 0.8505& 0.1621\\
        GAUSS         & 33.34 & 0.8950 & 0.0693 \\
        WIRE          & 35.38 &  0.9247 &0.0386\\
        MLP+PE+Ours& \textbf{40.96} & \textbf{0.9719} & \textbf{0.0126} \\
        MLP+RFF+Ours&\underline{36.63} & \underline{0.9545}& \underline{0.0220} \\
        \bottomrule
    \end{tabular}
    \label{tab:odak}
   \vspace{-5pt}
\end{table}
And the performance can be view at \autoref{figure:odak} where it can be found that our methods successfully reconstruct the windows with clarity, demonstrating their effectiveness.
\begin{figure*}[!ht]
    \centering
    \includegraphics[width=1.\textwidth]{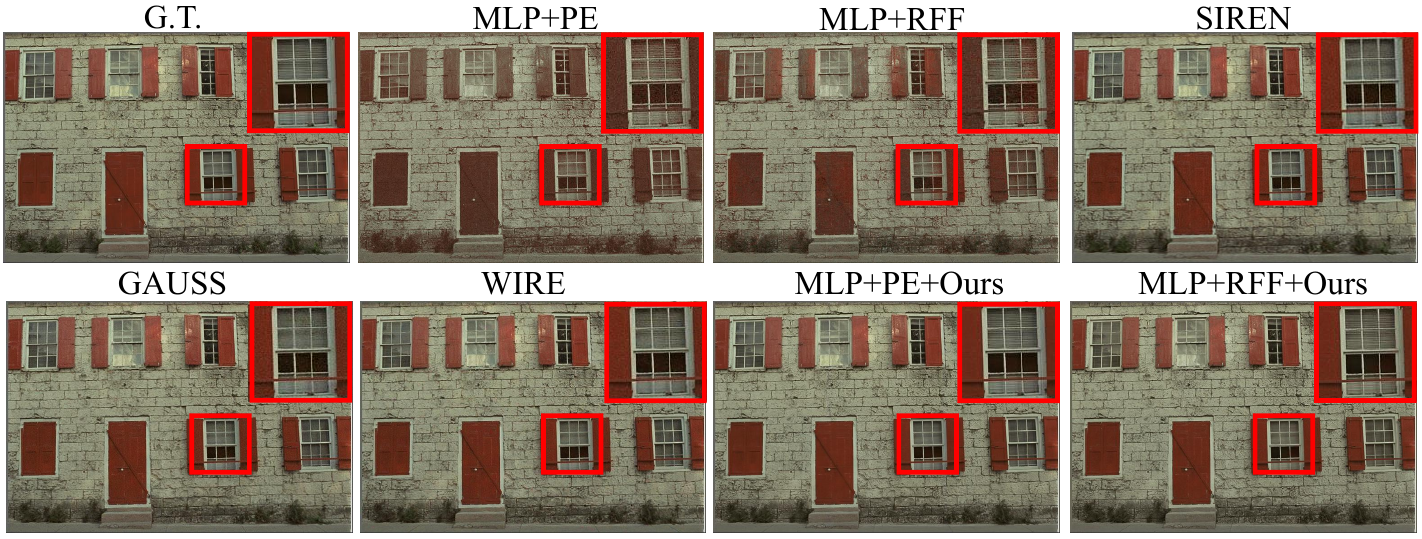}
    \caption{For the image regression task, our method can reach SOTA performance. It can be observed for the reconstruction quality of the window part in the image is the best without much noise and clear structure.}
    \vspace{-10pt} 
    \label{figure:odak}
\end{figure*}
Overall, by employing our method on the Fourier features embeddings, the overall performance can even surpass SOTA methods which validates the effectiveness of our proposed method that can not only reduce the noise level of the fitted result but also improve the fitting accuracy in different metrics.

\subsection{3D-Shape Regression}
\begin{table*}[!ht]
    \centering
    \caption{We highlight the best results in bold and underline the second-best results. Since the space is normalized, the detail difference will be extremely small in scale ($10^{-6}$) and the main structure is fitted well. However, the actual fitted result is quite different in details. Therefore, we also provide figures (\autoref{figure:exp_sdf}) of the performance to have a better understanding of the performance.}
    \resizebox{1\linewidth}{!}{
        \begin{tabular}{lcccccccc}
            \toprule
            Metric & MLP+PE & MLP+RFF & BACON & SIREN & GAUSS & WIRE & MLP+PE+Ours & MLP+RFF+Ours \\
            \midrule
            Chamfer Distance (↓) & 1.8413e-06 & 1.8525e-06 & 1.9535e-06 & 1.8313e-06 &  2.1593e-06&2.7243e-06&\textbf{1.7919e-06} & \underline{1.7947e-06} \\
            \bottomrule
        \end{tabular}}
    \label{table:metrics_comparison}
 \vspace{-10pt}
\end{table*}
 We evaluate our method on the Signed-Distance-Function (SDF) regression task, aiming to learn a function that maps 3D coordinates to their signed distance values. Positive values indicate points outside an object, and negative values are inside.
The objective is precise 3D shape reconstruction. We follow the experimental setup from \citet{lindell2022bacon}, training each model for 200,000 iterations with other hyperparameters the same as baselines provided. The learning rate for line-search was capped at $1 \times 10^{-3}$.
Performance is evaluated on four Stanford 3D Scanning Repository scenes~\footnote{\url{http://graphics.stanford.edu/data/3Dscanrep/}}: Armadillo, Dragon, Lucy, and Thai, each with 10,000 sampled points which is relatively sparse in 3D space.
To evaluate the performance of each model, we employ Chamfer distance instead of IOU. The absence of an official IOU implementation can lead to inconsistent results, whereas Chamfer Distance is computed by measuring the nearest vertex distances between the ground truth and predicted mesh vertices which can be implemented without much controversial and hyperparameters. Moreover, Chamfer distance reflects more details about the surface of object compared using sampling occupancy to calculate IOU. As the entire scene is rescaled to a 0–1 range, the Chamfer Distance is relatively small, with a magnitude on the order of \(10^{-6}\).
Our comparisons include baselines consistent with those used in the image regression task.

From quantification results shown in the \autoref{table:metrics_comparison}, Fourier features embeddings+our method achieves the lowest Chamfer Distance, demonstrating superior accuracy in shape reconstruction. Illustrations of results can be found at \autoref{figure:exp_sdf}, where it can be observed that the proposed method, to some extent, smoothed the surface while reconstructing more details compared with other baselines. However, GAUSS and WIRE tend to overfit the training set due to the sparsity of training set in 3D space, resulting in uneven surfaces, whereas SIREN exhibits smoother results but underfits the training dataset.

\begin{figure}[!ht]
    \includegraphics[width=.5\textwidth]{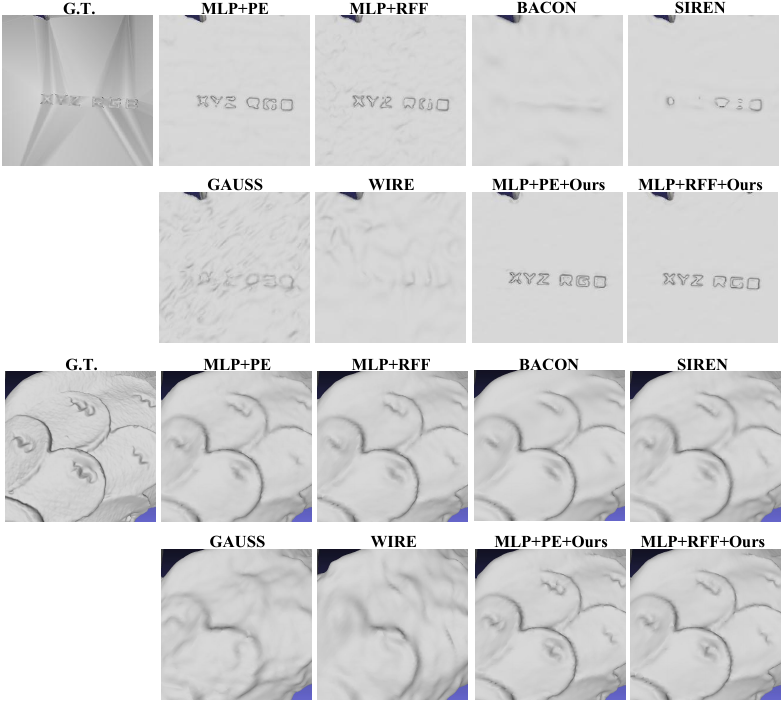} 
    \caption{Visualization of the 3D shape regression task shows that our method can smooth the surface while maintain detail structures.}
     \vspace{-10pt}
    \label{figure:exp_sdf}
\end{figure}

\subsection{Neural Radiance Field Experiments}
\begin{table}[!ht]
    \centering
    \caption{The quantitative results demonstrate our method can produce the best reconstruction in the dense input situations.}
    \vspace{10pt}
    \begin{tabular}{lccc}
        \toprule
        Methods & PSNR↑ & SSIM↑ & LPIPS↓ \\
        \midrule
        MLP+PE        & \underline{31.06} & \underline{0.9542} & \underline{0.0202} \\
        MLP+RFF        & 30.18 & 0.9476 & 0.0292 \\
        SIREN         & 25.52 & 0.8659 & 0.1500 \\
        GAUSS         & 27.87 & 0.9079 & 0.0707 \\
        WIRE          & 28.53 & 0.9198 & 0.0523 \\
        MLP+PE+Ours   & \textbf{31.45} & \textbf{0.9596} & \textbf{0.0172} \\
        MLP+RFF+Ours   & 30.70 & 0.9542 & 0.0232\\
        \bottomrule
        \vspace{-25pt}
    \end{tabular}
    \label{tab:nerf200}
    
\end{table}
This section explores the application of Neural Radiance Fields (NeRF) for fitting 3D scenes, focusing on reconstructing scenes by predicting color and density from 3D coordinates and viewing directions. The models are trained with MSE loss for 200,000 iterations, using the same hyperparameters as in the official implementation. Performance is evaluated using PSNR, SSIM, and LPIPS metrics. 

For the adaptive linear filer, we still employ 3 layers to maximize the performance. To minimize overfitting, we applied the line-search method (from $1 \times 10^{-3}$ to 0) and evaluated the models on the NeRF Blender dataset \citep{martin2021nerf} with 100 training images, which includes diverse synthetic scenes. Training utilized cropped 200$\times$200 images with a white background for consistency. Comparisons were conducted against a baseline MLP with Positional Encoding~\citep{mildenhall2021nerf}, SIREN~\citep{sitzmann2020implicit}, GAUSS~\citep{ramasinghe2022beyond} and WIRE~\citep{saragadam2023wire}. We use 8 layers for all models except for GAUSS and WIRE where we found 6 layers can achieve better performance. Since no official implementation of SIREN, GAUSS and WIRE exists for the NeRF task, we adapted the network structures implemented by WIRE to the nerf-pytorch codebase~\footnote{\url{https://github.com/yenchenlin/nerf-pytorch.git}} without altering hyperparameters that author suggested (notice that WIRE chooses sparse input setting but ours is dense setting).
The results in \autoref{tab:nerf200} show that our proposed method surpasses all baselines including WIRE and vanilla NeRF. As shown in \autoref{figure:nerflego}, our approach enables NeRF to capture finer details, such as the Lego's bucket.
\begin{figure}[!ht]
    \centering
    \includegraphics[width=0.45\textwidth]{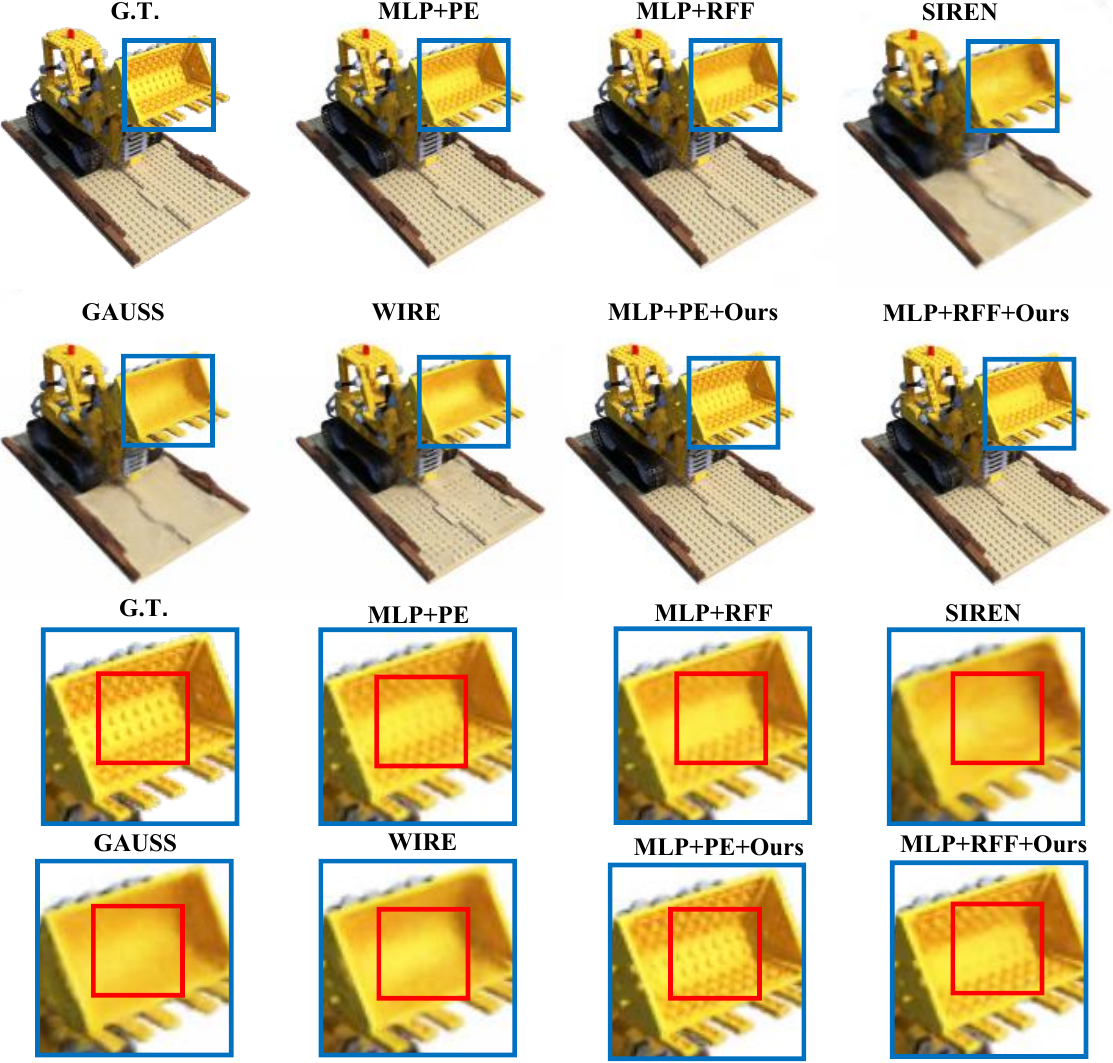}
    \caption{It can be observed that the reconstruction quality of the lego bucket remains high even at low resolutions when using our method.}
    \vspace{-15pt} 
    \label{figure:nerflego}
\end{figure}

\subsection{Ablation Study}
We validate the effectiveness of bias-free MLPs as adaptive linear filters still on the Kodak dataset. As even using our method learning rate scheduler is still necessary for the INRs part, therefore, we use Lambda Learning rate scheduler which is the most commonly used scheduler in all above tasks. The result is shown in the \autoref{tab:comparison}, where using our line-search based learning rate adjustor can further improve the performance.
\begin{table}[!ht]
    \centering
    \caption{Performance comparison of various methods for Image Regression. "w/o" stands for "without," "w/" stands for "with," and "L" refers to our custom line-search algorithm. Notice that our method is not contradict to the learning rate scheduler.}

    \small 
    \begin{tabular}{lccc}
        \toprule
        & \text{PSNR} $\uparrow$ & \text{SSIM} $\uparrow$ & \text{LPIPS} $\downarrow$ \\
        \midrule
        MLP + PE + Ours w/o L & 40.50 & 0.9691 &  0.0143 \\
        MLP + PE + Ours w/L  & \textbf{40.96} & \textbf{0.9719} & \textbf{0.0126} \\
        MLP + RFF + Ours w/o L & 36.08 & 0.9506 & 0.0247 \\
        MLP + RFF + Ours w/L  & \textbf{36.63} & \textbf{0.9545} & \textbf{0.0220}\\
        \bottomrule
    \end{tabular}
 
    \label{tab:comparison}
    \vspace{-10pt} 
\end{table}

\section{Conclusion}
In conclusion, we introduce a novel approach to reduce spectral bias and noise in implicit neural representations (INRs) with Fourier feature embeddings. By using bias-free MLPs as adaptive linear filters with line-search algorithm, our method suppresses unnecessary high frequencies and enhances embedding frequencies, boosting INRs performance.

\textbf{Limitations:} Despite the improvements, our method does not completely resolve finite sampling issues from the root. Additionally, while the line-search algorithm enhances the performance of the adaptive linear filter, it may lead to slower convergence. Addressing these challenges is part of our future work.
\bibliography{example_paper}
\bibliographystyle{icml2025}

\newpage
\appendix
\onecolumn
\clearpage
\section{Comparison with SAPE}
There might be arguing that our proposed method is similar to SAPE \citep{hertz2021sape}. However, from our perspective, we differ from this work in the following points:
\begin{itemize}
    \item Our proposed method can extend the frequency band of embeddings, i.e. $A_y*y$, where SAPE can be considered a simple mask that applied on the embeddings, i.e. $w*y$. This difference makes our method can reach better lower bound of the loss compared to SAPE. The performance on image regression task is demonstrate on \autoref{tab:SAPE}.
    
    \item  Our method using Bias-free MLPs as the filter which can be applied not only on grid-based structure, but also on continuous space like Neural Radiance Field. This is benefited from the continuous property of the prediction from MLPs.
\end{itemize}
\begin{table}[ht]
    \centering
        \caption{Performance comparison with SAPE and Our proposed method for both Positional Encoding and Random Fourier Features.}

    \small 
    \vspace{10pt}
    \begin{tabular}{lcccccc}
        \toprule
         &
        \multicolumn{3}{c}{Positional Encoding} &
        \multicolumn{3}{c}{Random Fourier Features} \\
        \cmidrule(lr){2-4} \cmidrule(lr){5-7}
        & PSNR$\uparrow$ & SSIM$\uparrow$ & LPIPS$\downarrow$ & PSNR$\uparrow$ & SSIM$\uparrow$ & LPIPS$\downarrow$ \\
        \midrule
        SAPE& 33.06 & 0.8996 & 0.0681 & 36.24 & 0.9455 & 0.0356  \\
        Ours  & \textbf{40.96} & \textbf{0.9719} & \textbf{0.0126} & \textbf{36.63} & \textbf{0.9545} & \textbf{0.0220}  \\
 
        \bottomrule
    \end{tabular}

    \label{tab:SAPE}
    \vspace{-10pt} 
\end{table}
\section{Convergence Comparison For Spike Function and Sinusoidal Function}
\label{sec:spike}
In this section, we compare the convergence speed and loss to provide a more profound understanding of how Fourier features, MLPs and their combination's representation capability for spike function and sinusoidal function (spatial compact and spectral compact) as illustrated in \autoref{figure:loss}.
\begin{figure}[!ht]
    \centering
    \includegraphics[width=0.75\textwidth]{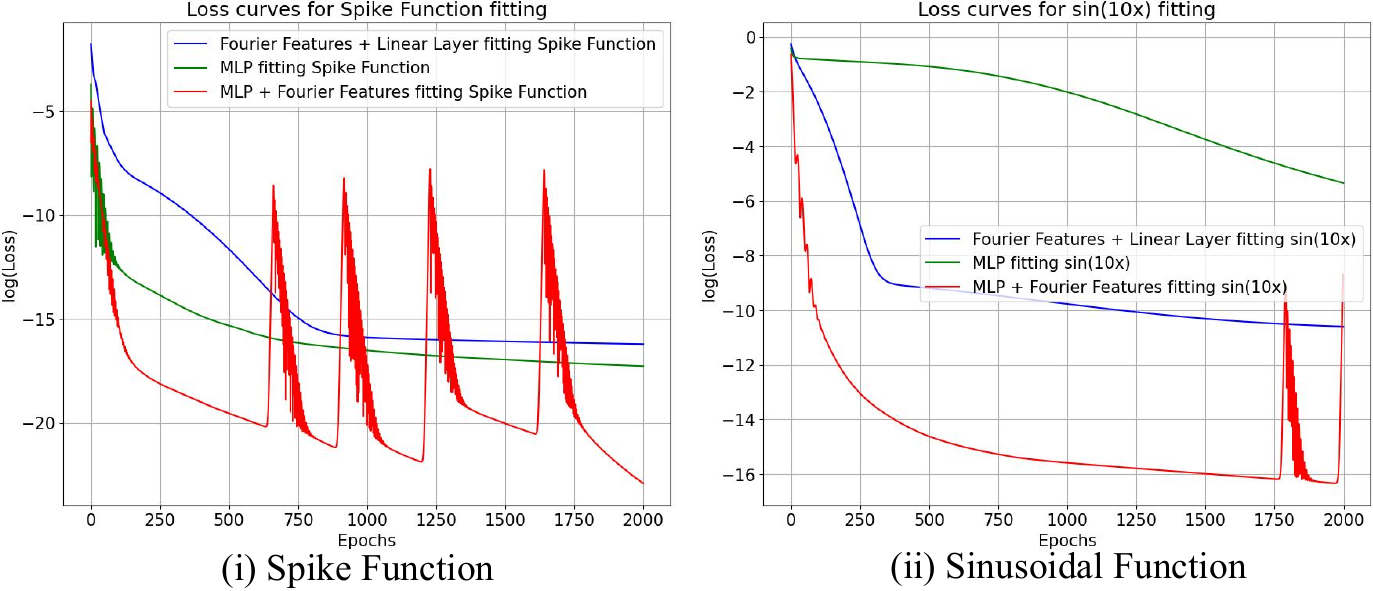}
    \caption{By comparing the value of loss function for Fourier features plus one linear layer, MLPs and their combination, it can be found that MLPs more good at fitting spatial compact function like spike function and Fourier features can fit sinusoidal function well. Through combining two models, the representation capability is even better for both situations.}
    \vspace{-10pt} 
    \label{figure:loss}
\end{figure}
\section{Convergence of Modified Line-search algorithm}
To address potential divergence concerns, we validate the convergence of the modified line-search algorithm on the DIV2K validation split with 256$\times$256 resolutions (for the fast inference speed and the large scale of the dataset(100 images)) for both RFF and PE embeddings. Results show consistent convergence for both, as illustrated in \autoref{figure:converg}. The learning rates of the adaptive linear filter steadily decrease throughout training, ultimately converging to 0, confirming the algorithm’s stability and convergence by the end of training.

\begin{figure}[!ht]
    \centering
    \includegraphics[width=0.75\textwidth]{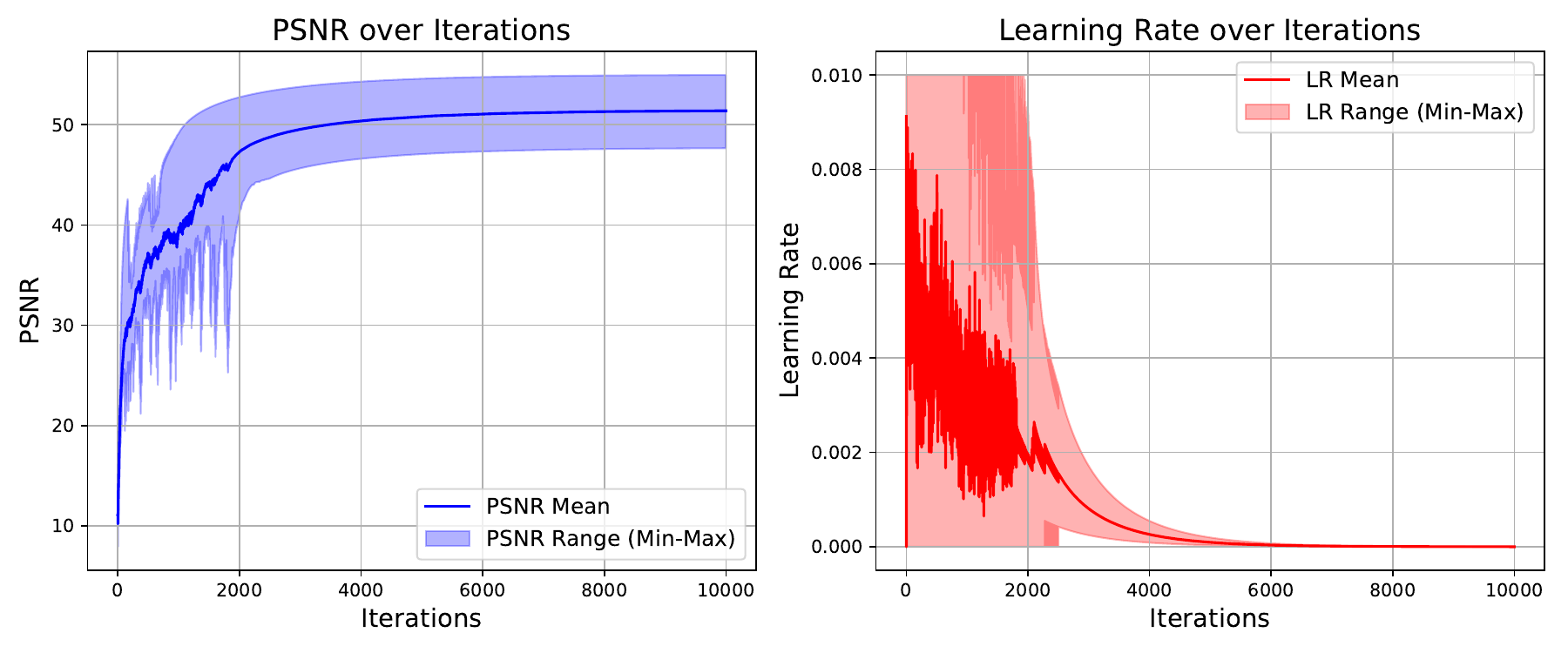}
    \caption{We demonstrate the convergence of the modified line-search algorithm through image regression experiments for RFF and PE. During training, both the PSNR and the learning rate consistently converge, confirming the effectiveness of our proposed line-search-based approach.}
    \vspace{-10pt} 
    \label{figure:converg}
\end{figure}
\section{Robustness under Varying Standard Deviation}
We also evaluate the impact of varying standard deviation on the same image regression task as in \autoref{figure:combined}. As shown in \autoref{figure:stats2}, unlike the results presented in \autoref{figure:combined}, performance remains stable even with high sampling standard deviation when our method is applied. This highlights the robustness of our approach under high sampling standard deviation.

\begin{figure}[!h]
    \centering
    \includegraphics[width=.8\textwidth]{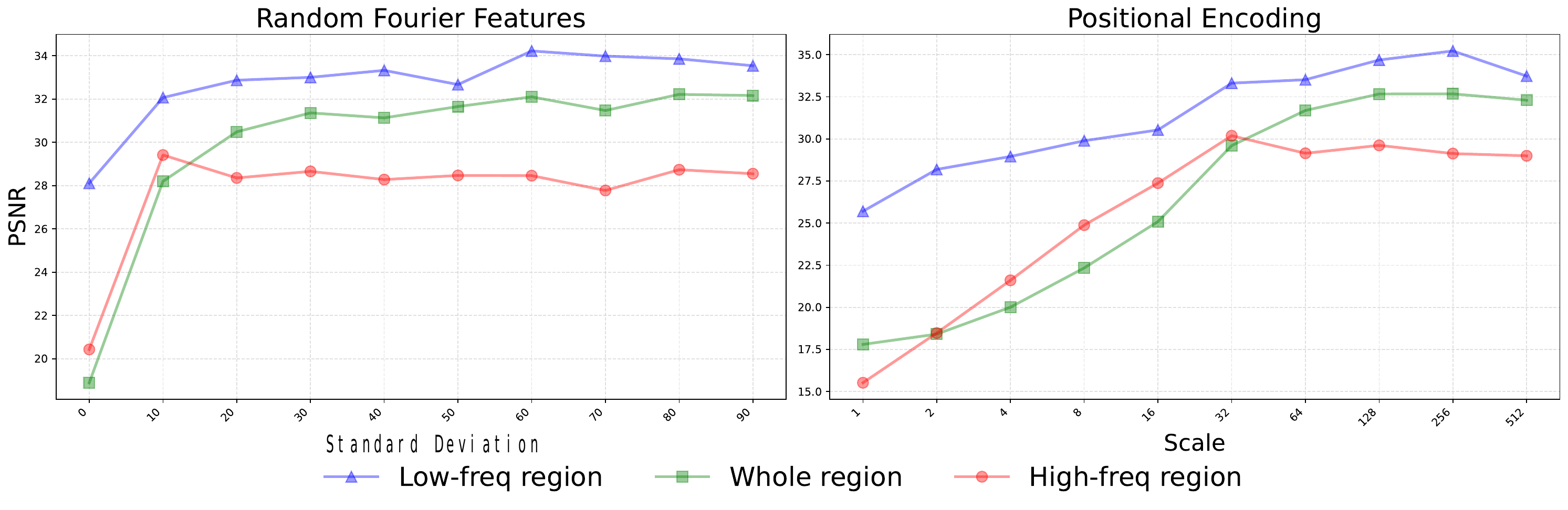}
    \caption{We evaluated our method's ability to mitigate high-frequency artifacts in two Fourier feature embedding methods. Results show that our approach effectively prevents model degradation under high standard deviation conditions for RFF, where traditional embeddings struggle.
}
    \vspace{-15pt} 
    \label{figure:stats2}
\end{figure}

\section{Varying the Number of Layers}
\label{section:layer}
In this section, we investigate the impact of the number of layers in bias-free MLPs on overall performance and demonstrate that simply increasing the number of layers of MLP plus Fourier features embeddings cannot significantly enhance performance as our methods. To evaluate this, we firstly conducted experiments on images from the DIV2K dataset at a resolution of 256$\times$256 for larger dataset and fast training speed. Filters with 1, 2, 3, 4, and 5 layers were tested with the same 3 layers INRs part, and the results are presented in \autoref{figure:bfMLPlayers}. The performance initially improves as the number of layers increases to 3 but starts to decline beyond that. This may be due to over-parameterization, which can reduce smoothness and make it more challenging for the filter to preserve the input signals.
\begin{figure}[!ht]
    \centering
    \includegraphics[width=0.4\textwidth]{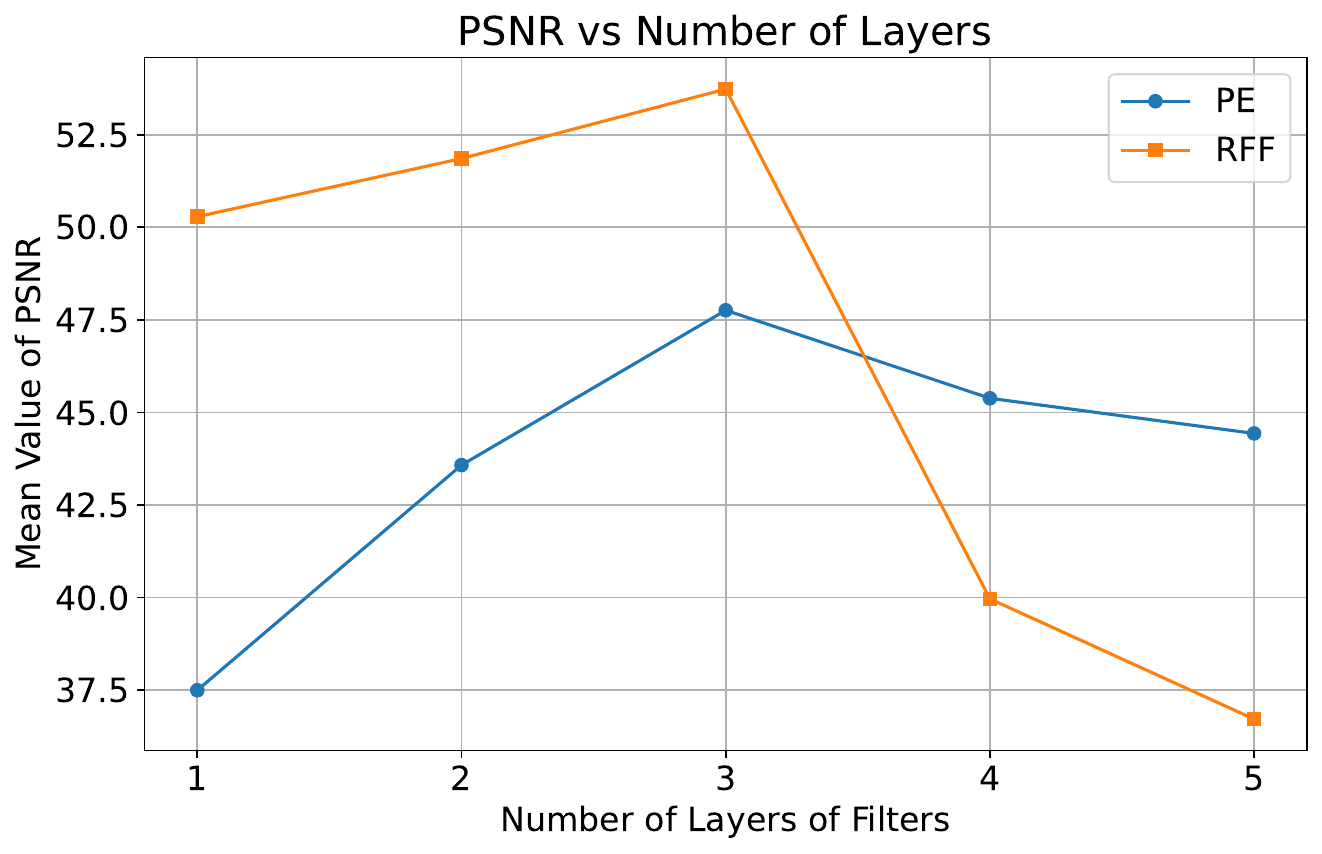}
    \caption{As the number of filter layers increases to 3, performance reaches its peak with around 52.7 PSNR and 47.5 PSNR respectively, but begins to decline with further increases in layers.}
    \vspace{-10pt} 
    \label{figure:bfMLPlayers}
\end{figure}
 We also test the impact of the number of layers on MLPs with Fourier features using the same sampled images. As shown in \autoref{figure:layers}, performance improves with more layers up to 12, after which it begins to decline. Despite the substantially larger number of parameters at 12 layers, the performance still fails to surpass that of our methods as previously illustrated in \autoref{figure:bfMLPlayers}.
\begin{figure}[!ht]
    \centering
    \includegraphics[width=0.4\textwidth]{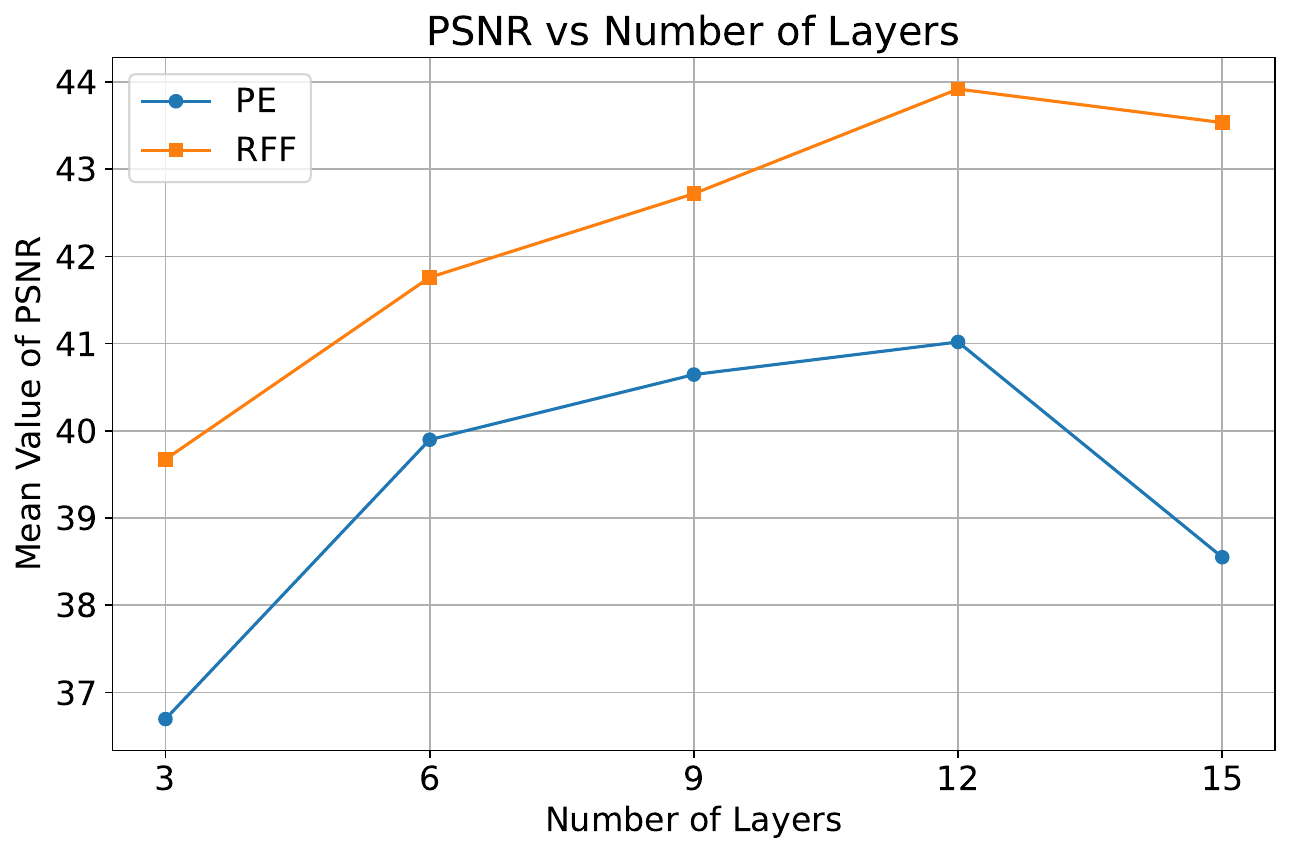}
    \caption{Performance increases with additional layers up to 12 with around 44 PSNR and 41 PSNR respectively, but starts to diminish beyond that point.}
    \vspace{-10pt} 
    \label{figure:layers}
\end{figure}

\section{Bias-free MLPs vs Bias MLPs}
We also compare the influence of additive term to the performance of the filter for the NeRF task. This is because that the NeRF task involves the interpolation task where the frequency pattern of the embeddings should not be disrupted by the filter. For the image regression task where it requires overfitting, the difference for the bias and bias-free MLPs may not be so obvious. The result is shown in \autoref{table:BIAS} where it shows that additive term indeed can worsen the performance.
\label{section:BFExp}
\begin{table}[!h]
\vspace{-10pt}
    \centering
    \caption{The result of adaptive linear filter w/w.o. bias term tested on the NeRF task.}
    \resizebox{0.75\linewidth}{!}{
    \begin{tabular}{lcccc}
        \toprule
      &  Bias MLP+PE &  Bias-free MLP+PE  & Bias MLP+RFF  &  Bias-free MLP+RFF \\
        \midrule
        PSNR↑  & 31.17 & \textbf{31.45}& 30.18&\textbf{30.70}  \\
        SSIM↑  &0.9563  &\textbf{0.9596}& 0.9476&\textbf{0.9542} \\
        LPIPS↓ & 0.0201 &\textbf{0.0172}& 0.0293&\textbf{0.0232}\\
        \bottomrule
    \end{tabular}
    }
    \label{table:BIAS}
    \vspace{-15pt}
\end{table}
\section{Definition of High-dimensional Fourier Series}
For a d-dimensional periodic function $f(\mathbf{x})$ with input $\mathbf{x} = [x_1, x_2,\cdots,x_d]^{\top}$ be a 2$\pi$ period function with respect to each components. Then the function $f(\mathbf{x})$ can be expanded as:
$$
f(\mathbf{x}) = \sum_{\mathbf{m}\in\mathbb{Z}^d}\hat{f}_m\mathbf{e}^{i\mathbf{m}^{\top}\mathbf{x}}
$$
where $\hat{f}_m$ is the coefficient of different frequency component.

\section{Definition of Neural Tangent Kernel}
\label{section：NTK}
The Neural Tangent Kernel (NTK), a prominent tool for neural network analysis, has attracted considerable attention since its introduction. To simplify the analysis, this section will focus specifically on the NTK for two-layer MLPs, as the subsequent analysis also relies on the two-layer assumption. The two-layer MLP, $f(\mathbf{x}; \mathbf{w})$, with activation function $\sigma(\cdot)$ and input $\mathbf{x} \in \mathbb{R}^d$, can be expressed as follows:
$$
    f(\mathbf{x}; \mathbf{w}) = \frac{1}{\sqrt{m}}\sum^{m}_{r=1}a_r\sigma(\mathbf{w}^{\top}_r\mathbf{x}+\mathbf{b}_r)
$$
where m is the width of the layer and $\Arrowvert\mathbf{x}\Arrowvert = 1$ (also can be written as $\mathbf{x}\in\mathbb{S}^{d-1}$, where $\mathbb{S}^{d-1}\equiv\{\mathbf{x}\in\mathbb{R}^d:\Arrowvert\mathbf{x}\Arrowvert = 1\}$). The term $\frac{1}{\sqrt{m}}$ is used to assist the analysis of the network. Based on this MLP, the kernel is defined as the following:
$$
k(\mathbf{x_i}, \mathbf{x_j}) = \mathbb{E}_{\mathbf{w}\sim\mathcal{I}}\left\lbrace\left\langle\frac{\partial f(\mathbf{x_i};\mathbf{w})}{\partial \mathbf{w}}, \frac{\partial f(\mathbf{x_j};\mathbf{w})}{\partial \mathbf{w}}\right\rangle\right\rbrace
$$
This formula enables the exact expression of the NTK to better analyze the behavior and dynamics of MLP. 
For a two-layer MLP with a rectified linear unit (ReLU) activation function where only the first layer weights are trained and the second layer is frozen, the NTK of this network can be written as the following \citep{xie2017diverse}:
\begin{equation*}
k(\mathbf{x_i}, \mathbf{x_j}) = \frac{1}{4\pi}(\langle\mathbf{x_i}, \mathbf{x_j}\rangle+1)(\pi-arccos(\langle\mathbf{x_i}, \mathbf{x_j}\rangle))    
\label{Eq:NTK}
\end{equation*}
This expression can help us to determine the eigenfunction and eigenvalue of kernel and therefore provide a more insightful analysis of the network.
\section{Proof of the Proposed Proposition}
\label{section:proof}
In this section, we will introduce why the unselected frequencies of the Fourier features will form a lower bound for the theoretical performance. Compared to \citep{yuce2022structured} where they show that the INRs with embeddings can be decomposed into the Fourier basis, we view the INRs with Fourier features embeddings from the perspective of Neural Tangent Kernels and derive a similar result about the Harmonic expansion of the INRs.

\begin{lemma}
    (\citet{yuce2022structured}) Let $\{\mathbf{b}^{(1)}_i\in\mathbb{R}^d\}_{i\in[N]}$ and $\{\mathbf{b}^{(2)}_j\in\mathbb{R}^d\}_{j\in[M]}$ be two sets of frequency vectors and N and M are integers that represent the size for each set, $\mathbf{x}\in\mathbb{R}^d$ is the coordinates in d-dimensional space. Then,
    \begin{equation*}
    \begin{split}
            &\left(\sum^N_{i=1}c^{(1)}_i cos(\mathbf{b}^{(1)\top}_i\mathbf{x})\right)\left(\sum^M_{j=1}c^{(2)}_jcos(\mathbf{b}^{(2)\top}_j\mathbf{x})\right) \\&= \left(\sum^{T}_{k=1}c^{*}_kcos(\mathbf{b}^{*\top}_k\mathbf{x})\right), where\, T\leq 2NM
    \end{split}
    \end{equation*}
    where,
    \begin{equation*}
    \mathbf{b}^*\in \left\{ \mathbf{b}^*=\mathbf{b}^{(1)}_i\pm\mathbf{b}^{(2)}_j\,\vline\,i\in[N],\, j\in[M]\right\}
    \end{equation*}
\label{lemma:2}
\end{lemma}
\begin{lemma}
    (\citet{yuce2022structured}) $\{\mathbf{b}_i\in\mathbb{R}^d\}_{i\in[n]}$ be a set of frequency vectors and N is an integer that represents the size, $\mathbf{x}\in\mathbb{R}^d$ is the coordinates in d-dimensional space. Then,
    \begin{equation*}
    \left(\sum^n_{i=1}cos(\mathbf{b}^{\top}_i\mathbf{x})\right)^k = \left(\sum^{N}_{k=1}cos(\mathbf{b}^{*\top}_k\mathbf{x})\right), where\, N \leq k^nnk
    \end{equation*}
    where,
    \begin{equation*}
    \mathbf{b}^*\in \left\{ \mathbf{b}^*=\sum_i^n c_i\mathbf{b_i}\,\vline\, c_i\in\mathbb{Z},\,\sum_i^n|c_i|\leq k\right\}
    \end{equation*}
\label{lemma:3}
\end{lemma}


\begin{lemma}
      Given a pre-sampled frequency set $\mathbf{B}_n = \{\mathbf{b}_i\in\mathbb{N}^d\}_{i\in[N]}$ and the Fourier features projection, $\gamma(\cdot)$, as $\gamma(\mathbf{x}) = [sin(2\pi \mathbf{b}_i^{\top}\mathbf{x}), cos(2\pi \mathbf{b}_i^{\top}\mathbf{x})]_{i\in[N]}, [N] = {1, 2, 3,\cdots, N}$. Then, $\gamma(\mathbf{x})^\top\gamma(\mathbf{
      z}) = sum(\gamma(\mathbf{x-z}))$.
\label{lemma:4}
\end{lemma}
\begin{proof}
    \begin{align*}
        &\gamma(\mathbf{x})^\top\gamma(\mathbf{z})\\ &= \sum^N_{i=1} cos(2\pi \mathbf{b}_i^{\top}\mathbf{x})cos(2\pi \mathbf{b}_i^{\top}\mathbf{z}) + sin(2\pi \mathbf{b}_i^{\top}\mathbf{x})sin(2\pi \mathbf{b}_i^{\top}\mathbf{z})\\
        &=\sum^N_{i=1}cos(2\pi \mathbf{b}_i^{\top}(\mathbf{x}-\mathbf{z})) = sum(\gamma(\mathbf{x}-\mathbf{z}))
    \end{align*}
\end{proof}
\begin{lemma}
    For a two-layer Multilayer-perceptrons (MLPs) denoted as $f(\mathbf{x};\mathbf{W})$, where $\mathbf{x}\in\mathbb{R}^d$ as input and $\mathbf{W}$ as the parameters of the MLPs. 
    Then the order-N approximation of eigenvectors of the Neural Tangent Kernel (Eq.\ref{Eq:NTK}) when using Fourier features embedding, as defined in Def.\ref{def:ff}, to project the input to the frequency space can be presented as,
    \begin{equation*}
    \begin{split}
               &k(\gamma(\mathbf{x}),\gamma(\mathbf{z}))= \sum^{N^{\dag}}_{i=1} \lambda^2_i cos(\mathbf{b}^*\mathbf{x})cos(\mathbf{b}^*\mathbf{z})\\ &+ \sum^{N^{\dag}}_{i=1} \lambda^2_i sin(\mathbf{b}^*\mathbf{x})sin(\mathbf{b}^*\mathbf{z})\text{, where $N^\dag\leq 4Nk^mkm^2$}
    \end{split}
    \end{equation*}
    where 
\begingroup
\small 
\begin{equation*}
    \mathbf{b}^*\in \mathcal{L}_{Span\{b_j\}} \equiv \left\{   \mathbf{b}^*=\sum^n_{j=1} c_j\mathbf{b}_j\,\vline\, \sum^{\infty}_{j=1}|c_j|<N+k^mkm+m  \right\}
\end{equation*}
\endgroup
    and $\lambda_i$s are eigenvalues for each eigenfunctions $sin(\mathbf{b}^*\mathbf{x})$ and $cos(\mathbf{b}^*\mathbf{x})$.
\label{lemma1}
\end{lemma}
\begin{proof}
    By ~\citet{xie2017diverse}, the two-layer MLP's NTK has the form as the following:
    \[
    k(x,z)= \frac{\langle\mathbf{x},\mathbf{z}\rangle(\pi - arccos(\langle\mathbf{x},\mathbf{z}\rangle)}{2\pi}
    \]
    If we use Fourier features mapping, $\gamma(\mathbf{x})$, before inputting to the Neural Network with a randomly sampled frequency set $\{\mathbf{b}_i\}^m_{i=1}$.
    
    By the Lemma \ref{lemma:4}, in order to ensure that the vector dot product still be a valid dot product in $S^{d-1}$, the dot product of two embedded input can be written as $\gamma(\mathbf{x})^{\top}\gamma(\mathbf{z})=\frac{1}{||\gamma(\mathbf{x})||||\gamma(\mathbf{z})||}\sum^m_{i=1}cos(2\pi\mathbf{b}_i(\mathbf{z}-\mathbf{x}))$ to make sure the dot product is bounded by 1.

    \begin{align*}
         k(&\gamma(\mathbf{x}),\gamma(\mathbf{z}))= \frac{\langle\gamma(\mathbf{x}),\gamma(\mathbf{z})\rangle(\pi - arccos(\langle\gamma(\mathbf{x}),\gamma(\mathbf{z})\rangle)}{2\pi}\\
         &\textbf{Denoting $||\gamma(\mathbf{x})||||\gamma(\mathbf{z})||$ as $\aleph$}\\
         &=\frac{\sum^m_{i=1}cos(2\pi\mathbf{b}_i(\mathbf{z}-\mathbf{x}))(\pi - arccos(\frac{1}{\aleph}\sum^m_{i=1}cos(2\pi\mathbf{b}_i(\mathbf{z}-\mathbf{x}))))}{2\pi\aleph}\\
         &\textbf{By N-order approximation Taylor Expansion of arccos($\cdot$)}\\
         &=\frac{1}{2\pi\aleph}(\sum^m_{i=1}cos(2\pi\mathbf{b}_i(\mathbf{z}-\mathbf{x}))\times\\&(\frac{\pi}{2}+\sum^{N}_{k=1}\frac{(2n)!}{2^{2n}} (n!)^2
(\sum^m_{i=1}\frac{1}{\aleph}cos(2\pi\mathbf{b}_i(\mathbf{z}-\mathbf{x})))^k))\\
         &\textbf{By Lemma \ref{lemma:4}}\\
         &=\frac{\sum^m_{i=1}cos(2\pi\mathbf{b}_i(\mathbf{z}-\mathbf{x}))(\frac{\pi}{2}+\sum^{N}_{k=1}\sum^M_{i=1}\beta^*_icos(2\pi\mathbf{b}^*_i(\mathbf{z}-\mathbf{x})))}{2\pi\aleph}\\& \text{where M$\leq k^mkm$ }\\
         &\text{and}\,\,\mathbf{b}^*_i\in \left\{ \mathbf{b}^*=\sum_i^m c_i\mathbf{b_i}\,\vline\, c_i\in\mathbb{Z},\,\sum_i^n|c_i|\leq k\right\}\\
          &=\frac{\sum^m_{i=1}cos(2\pi\mathbf{b}_i(\mathbf{z}-\mathbf{x}))(\frac{\pi}{2}+\sum^{N^*}_{i=1}\beta^*_icos(2\pi\mathbf{b}^*_i(\mathbf{z}-\mathbf{x})))}{2\pi\aleph}\\ &\text{where $N^*\leq 2NM$ }\\
         &\text{and}\,\,\mathbf{b}^*_i\in \left\{ \mathbf{b}^*=\sum_i^m c_i\mathbf{b_i}\,\vline\, c_i\in\mathbb{Z},\,\sum_i^
         n|c_i|\leq N+M\right\}\\
         &=\frac{1}{2\pi\aleph}(\frac{\pi}{2}\sum^m_{i=1}cos(2\pi\mathbf{b}_i(\mathbf{z}-\mathbf{x}))+\sum^m_{i=1}cos(2\pi\mathbf{b}_i(\mathbf{z}-\mathbf{x}))\times\\ &\sum^{N^*}_{i=1}
         \beta^*_icos(2\pi\mathbf{b}^*_i(\mathbf{z}-\mathbf{x}))))\\
         &\textbf{By Lemma \ref{lemma:3}}\\
         &=\frac{\frac{\pi}{2}\sum^m_{i=1}cos(2\pi\mathbf{b}_i(\mathbf{z}-\mathbf{x}))+\sum^{N^{\dag}}_{i=1}
         \beta^\dag_icos(2\pi\mathbf{b}^\dag_i(\mathbf{z}-\mathbf{x})))}{2\pi\aleph}\\ &\text{where $N^\dag\leq 2mN^*$ }\\
         &\text{and}\,\,\mathbf{b}^\dag_i\in \left\{ \mathbf{b}^\dag=\sum_i^m c_i\mathbf{b_i}\,\vline\, c_i\in\mathbb{Z},\,\sum_i^n|c_i|\leq N+M+m\right\}\\
      &=\frac{1}{4\aleph}\sum^m_{i=1}cos(2\pi\mathbf{b}_i(\mathbf{z}-\mathbf{x}))
      \\&+\frac{1}{2\pi\aleph}\sum^{N^{\dag}}_{i=1}
\beta^\dag_icos(2\pi\mathbf{b}^\dag_i(\mathbf{z}-\mathbf{x})))\text{, where $N^\dag\leq 4Nk^mkm^2$}\\
    \end{align*}
 
    Furthermore, to do the eigendecomposition, we need further to split this into the product of two orthogonal functions by $cos(a-b)=cos(a)cos(b) + sin(a)sin(b)$\\
    
    \begin{align*}
                  =&\frac{1}{4\aleph}\sum^m_{i=1}cos(2\pi\mathbf{b}_i\mathbf{x}))cos(2\pi\mathbf{b}_i\mathbf{z}))+sin(2\pi\mathbf{b}_i\mathbf{x})sin(2\pi\mathbf{b}_i\mathbf{z})
          \\
          &+\frac{1}{2\pi\aleph}\sum^{N^{\dag}}_{i=1}
\beta^\dag_icos(2\pi\mathbf{b}^\dag_i\mathbf{x}))cos(2\pi\mathbf{b}^\dag_i\mathbf{z}))+sin(2\pi\mathbf{b}^\dag_i\mathbf{x})sin(2\pi\mathbf{b}^\dag_i\mathbf{z})
    \end{align*}

\end{proof}
This lemma explains why MLPs plus Fourier features embeddings can be interpreted as the linear combination of sinusoidal functions which further provide the evidence of the following lemma \ref{lemma1}.
\begin{theorem}
[Theorem 4.1 in \citet{arora2019fine}] Denoting $\mathbf{u}^{(k)}$ as the prediction of MLPs at iteration $k$, $\lambda$ as the eigenvalue of MLPs, $\mathbf{v}_i$ as the eigenfunction of MLPs, m as the number of neurons in a single layer and $\eta$ as the learning rate. Suppose $\lambda_0 = \lambda_{\min}(\mathbf{H}^\infty) > 0$, 
$\kappa = O\left(\frac{\epsilon \delta}{\sqrt{n}}\right)$, 
$m = \Omega\left(\frac{n^7}{\lambda_0^6 \delta^2 \epsilon^2}\right)$, 
and $\eta = O\left(\frac{\lambda_0}{m}\right)$. Then with probability 
at least $1 - \delta$ over the random initialization, for all $k = 0, 1, 2, \dots$, we have:
\[
\|\mathbf{y} - \mathbf{u}^{(k)}\|_2 = 
\sqrt{\sum_{i=1}^n (1 - \eta \lambda_i)^{2k} (\mathbf{v}_i^\top \mathbf{y})^2} \pm \epsilon.
\]
\label{Thmeorem:brought}
\end{theorem}
\begin{proof}
    Check the proof of Theorem 4.1 in \citet{arora2019fine}
\end{proof}
Notice that the above theorem is based on the full-batch training assumption which is widely used for the image regression task where there is no interpolation requirement.
\begin{lemma}
Let $\mathbf{y}(\mathbf{x}) = \sum_{\mathbf{n}\in\mathbb{Z}^d} \hat{y}_{\mathbf{n}} e^{i \mathbf{n}^\top \mathbf{x}}$ be a $d$-dimensional target function, where $\hat{y}_{\mathbf{n}}$ denotes the Fourier coefficients of $\mathbf{y}(\mathbf{x})$. Consider a pre-sampled frequency set $\mathbf{B}_n = \{\mathbf{b}_i \in \mathbb{Z}^d\}_{i \in [N]}$ and the $L_2$ loss function defined as $\phi(\mathbf{y}, f(\mathbf{x}; \mathbf{W})) = \|f(\mathbf{x}; \mathbf{W}) - \mathbf{y}\|_2$. Let $f(\mathbf{x}; \mathbf{W})$ denotes the MLPs that can be written in the form of sum of sinusoidal functions with frequencies in subspace spanned by $\mathbf{B}_n$ using Neural Tangent Kernel expansion, $\mathbf{y}_\mathbf{B}$ denote the mapping of $\mathbf{y}(\mathbf{x})$ onto the subspace spanned by $\mathbf{B}_n$ which is the same as Neural Tangent Kernel expansion of MLPs, and $\mathbf{y}^\dag_\mathbf{B}$ denote the mapping onto the orthogonal complement, such that $\mathbf{y} = \mathbf{y}_\mathbf{B} + \mathbf{y}^\dag_\mathbf{B}$. Then, with probability at least $1 - \delta$, for all iterations $k = 0, 1, 2, \dots$, the lower bound of the loss function $\phi(\mathbf{y}, f(\mathbf{x}; \mathbf{W}))$ can be expressed as:
\begin{equation*}
    \phi(\mathbf{y}, f(\mathbf{x}; \mathbf{W})) \geq \|\mathbf{y}^\dag_\mathbf{B}\|_2 - \sqrt{\sum_{i}(1 - \eta \lambda_i)^{2k} \langle \mathbf{v}_i, \mathbf{y}_\mathbf{B} \rangle^2} \pm \epsilon,
\end{equation*}
where $\eta$ is the learning rate, $\lambda_i$ are eigenvalues, $\mathbf{v}_i$ are the corresponding eigenvectors, and $\epsilon$ stands for a constant.
\label{lemma:1}
\end{lemma}
\begin{proof}
    Given $\mathbf{B}_n = \{\mathbf{b}_i\in\mathbb{N}^d\}_{i\in[N]}$, this can be spanned as a subspace, $\{cos(2\pi\mathbf{b}^\dag\mathbf{x}), sin(2\pi\mathbf{b}^\dag\mathbf{x})|\mathbf{b}^\dag_i\in\{ \mathbf{b}^\dag=\sum_i^m c_i\mathbf{b_i}\,\vline\, c_i\in\mathbb{Z}$, $\,\sum_i^n|c_i|\leq N^\dag\}\}$, in $\mathcal{L}_d$ space, where each item are orthogonal with each other. Therefore, by Orthogonal Decomposition Theorem, $\mathcal{L}_d$ can be decomposed into the space spanned by $\mathbf{B}_n$ and the space orthogonal to this spanned space (or the space spanned by the rest frequencies component).
    
    By using the Fourier series to expand the $y$, this can be decomposed into $y =y^{\dag}_\mathbf{B}+y_\mathbf{B} $ by orthogonal decomposition theorem mentioned before. And each $\mathbf{v}_i$, the eigenfunction of $f(\mathbf{x}; \mathbf{W})$, belongs to the the spanned subspace by $\mathbf{B}_n$ (by Lemma \ref{lemma1}). Therefore, $f(\mathbf{x}; \mathbf{W})$ can be written as a linear combination of Fourier bases and, therefore, cannot fit signals in the orthogonal space of the spanned subspace of $\mathbf{B}_n$ due to orthogonality.
    \begin{align*}
        &\phi(\mathbf{y}, f(\mathbf{x};\mathbf{W})) = ||\mathbf{y}- f(\mathbf{x};\mathbf{W})||_2\\
        &= ||\mathbf{y}^{\dag}_\mathbf{B}+\mathbf{y}_\mathbf{B}- f(\mathbf{x};\mathbf{W})||_2\\
        &\textbf{By using triangular inequality:$||x+y|| - ||x||\leq||\mathbf{y}||$}\\
        &\geq ||\mathbf{y}^{\dag}_\mathbf{B}||_2 -||f(\mathbf{x};\mathbf{W})-\mathbf{y}_\mathbf{B}||_2 
    \end{align*}
    By Theorem \ref{Thmeorem:brought}, with probability $1-\delta$, $\phi(\mathbf{y}, f(\mathbf{x};\mathbf{W})) = \sqrt{\sum(1-\eta\lambda_i)^{2k}\langle\mathbf{v}_i,\,\mathbf{y}\rangle^2}\pm\epsilon$. We can only decompose the latter part and obtain the proposed result.
\end{proof}

\section{Line-search method}
In this section, we firstly make some definition of notations:

\begin{table}[!ht]
\centering
\begin{tabular}{@{}ll@{}}
\toprule
\textbf{Symbol} & \textbf{Explanation} \\ 
\midrule
$\mathbf{X}$     & Dataset pair ($\mathbf{x}$, $\mathbf{y}$) \\
$\theta^t\in\Theta$     &Parameters of MLPs at iteration t\\
$\theta^t_I\in\Theta$     &Parameters of INRs at iteration t \\
$\theta^t_A\in\Theta$     &Parameters of filters at iteration t \\
$\theta^*\in\Theta$     &The optimal parameters of MLPs \\
$\alpha$         & Learning rate \\ 
$\alpha_I$         & Learning rate of INRs\\ 
$\alpha_A$         & Learning rate of filters\\ 
$\phi(\cdot)$  & Loss function that can depend on $\alpha$, \\&$\mathbf{X}$ and $\theta$\\
$\nabla $       & Gradient operator \\ 
$\mathbf{p}^t$ & The update direction at iteration t\\
$\mathbf{p}^t_I$ & The update direction of INRs at iteration t\\
$\mathbf{p}^t_A$ & The update direction of filters at iteration t\\
\bottomrule
\end{tabular}
\label{tab:notations}
\end{table}

Considering a minimization problem as the following:
\begin{equation*}
\begin{aligned} \label{P}
&\theta^* = \min_{\theta\in\Theta}\phi(\theta) \\
\end{aligned}
\end{equation*}
 One common method to find the $\theta^*\in\Theta$ that minimizes $\phi(\cdot)$ is to use the gradient descent method as shown in the Algorithm \ref{GD}.

\begin{algorithm}
\caption{Gradient Descent Algorithm}
\begin{algorithmic}[1] 
    \STATE \textbf{Initialize:} variables $\theta^0$, max iteration $N$, learning rate $\alpha$
    \FOR{$i \gets 1$ to $N$}
        \STATE Calculate the derivative of $\phi(\theta^{t-1})$ about $\theta^{t-1}$ as direction denotes as $\mathbf{p}^t$
        \STATE $\theta^t\gets\theta^{t-1}+\alpha\mathbf{p}^t$
    \ENDFOR
\end{algorithmic}
\label{GD}
\end{algorithm}

Based on this Gradient Descent method, the line-search method is to find the proper learning rate $\alpha_t$ at each iteration by optimization to solve the approximate learning rate or exact learning rate if possible. The algorithm can be shown as the Algorithm \ref{LS}.
\begin{algorithm}
\caption{Line-search Method}
\begin{algorithmic}[1] 
    \STATE \textbf{Initialize:} variables $\theta^0$, max iteration $N$
    \FOR{$i \gets 1$ to $N$}
        \STATE Calculate the derivative of $\phi(\theta^{t-1})$ about $\theta_{t-1}$ as direction denotes as $\mathbf{p}^t$
        \STATE $\alpha=arg\,min_{\alpha}\phi(\theta^{t-1}+\alpha \mathbf{p}^t)$ 
        \STATE $\theta^t\gets\theta^{t-1}+\alpha\mathbf{p}^t$
    \ENDFOR
\end{algorithmic}
\label{LS}
\end{algorithm}
\subsection{Details of Custorm Line-Search Algorithm}
\begin{figure*}[!ht]
    \centering
    \includegraphics[page=1, width=.95\textwidth]{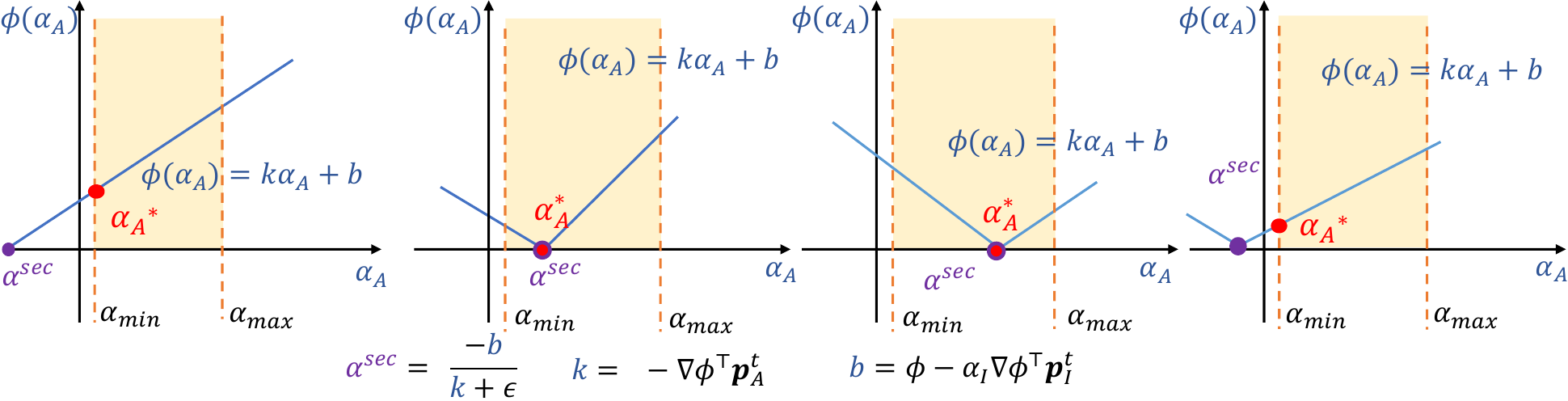}
    \caption{The blue line is the optimization target, while the orange lines indicate the predefined learning rate bounds, denoted as $\alpha_{\text{min}}$ and $\alpha_{\text{max}}$. $\mathbf{p}^t_A$ and $\mathbf{p}^t_I$ are the update directions for the filter and INRs, respectively. $\alpha^*$ is the optimal value and $\epsilon$ is a constant for robustness, usually set to $1 \times 10^{-6}$.
}
    \label{figure:line-search}
\end{figure*}
\label{section:LSA}
In this section, we will explain the derivation of the modified line-search algorithm used to determine the learning rate of the adaptive filter.

Let $\phi(\theta^t_A,\theta^t_I)$ denote the loss function at iteration $t$, $\mathbf{p}^t_A$ as the update direction for the adaptive filter, and $\mathbf{p}^t_I$ as the update direction for the INRs. We can then perform a Taylor expansion around the parameters $(\theta^{t-1}_A, \theta^{t-1}_I)$, expressed as follows:
\begingroup
\scriptsize
\[
\phi(\theta^{t}_A, \theta^{t}_I) = \phi(\theta_A^{t-1}, \theta_I^{t-1}) - \nabla_{\theta_A^t} \phi(\theta_A^{t-1}, \theta_I^{t-1})^\top (\theta_A^t - \theta_A^{t-1})\]\[ - \nabla_{\theta_I^t} \phi(\theta_A^{t-1}, \theta_I^{t-1})^\top (\theta_I^t - \theta_I^{t-1})
\]
\[
+ \frac{1}{2} \left[ (\theta_A^t - \theta_A^{t-1})^2 \triangle_{\theta_A^t} \phi(\theta_A^{t-1}, \theta_I^{t-1}) + (\theta_I^t - \theta_I^{t-1})^2\triangle_{\theta_I^t} \phi(\theta_A^{t-1}, \theta_I^{t-1}) \right]\]
\[+\mathcal{O}((\theta_I^t - \theta_I^{t-1})^2, (\theta_A^t - \theta_A^{t-1})^2)
\]
\endgroup
Using the gradient descent method, we have $\theta^{t} = \theta^{t-1} + \alpha p^{t-1}$. Since the ReLU activation function results in the second and higher-order derivatives being zero, the equation simplifies to:
\begingroup
\scriptsize
\[
\phi(\theta^{t}_A, \theta^{t}_I) \approx \phi(\theta_A^{t-1}, \theta_I^{t-1}) - \nabla_{\theta_A^t} \phi(\theta_A^{t-1}, \theta_I^{t-1})^\top (\alpha_{A}\mathbf{p}^{t-1}_A)\]\[ - \nabla_{\theta_I^t} \phi(\theta_A^{t-1}, \theta_I^{t-1})^\top (\alpha_{I}\mathbf{p}^{t-1}_I)
\]
\[
\Rightarrow\phi(\alpha_A) = \phi(\theta^{t}_A, \theta^{t}_I) \approx k\alpha_A+b
\]
\[\mathbf{where}\,\, k = -\nabla_{\theta_A^t} \phi(\theta_A^{t-1}, \theta_I^{t-1})^\top \mathbf{p}^{t-1}_A\]\[\mathbf{and}\, b =  \phi(\theta_A^{t-1}, \theta_I^{t-1}) - \nabla_{\theta_I^t} \phi(\theta_A^{t-1}, \theta_I^{t-1})^\top (\alpha_{I}\mathbf{p}^{t-1}_I)\]
\endgroup
Since the learning rate of the INRs part is known, this can be simplified as a linear optimization problem with only an order 1 unknown parameter $\alpha^t_A$.
\begingroup
\scriptsize
\begin{equation*}
    \begin{split}
&\arg\min_{\alpha_A} \phi(\theta^{t}_A, \theta^{t}_I) \approx\\ &\arg\min_{\alpha_A} \phi(\theta_A^{t-1}, \theta_I^{t-1}) - \nabla_{\theta_A^t} \phi(\theta_A^{t-1}, \theta_I^{t-1})^\top (\alpha_{A}\mathbf{p}^{t-1}_A)\\
&- \nabla_{\theta_I^t} \phi(\theta_A^{t-1}, \theta_I^{t-1})^\top (\alpha_{I}\mathbf{p}^{t-1}_I)
    \end{split}
\end{equation*}

\endgroup
To mitigate the impact of a small denominator, a constant $\epsilon=1\times10^{-6}$ is introduced, ensuring the robustness of the algorithm. The solution to this optimization problem is derived through a case-based analysis by treating the loss function as a function of $\alpha_A$, since $\theta^{t}_A$ is expressed as a function of $\theta^{t-1}_A$ and $\alpha_A$. Specifically, the loss function, denoted as $\phi(\alpha_A) = \phi(\theta^{t}_A, \theta^{t}_I)$. The analysis involves examining the sign of the slope and intercept of the loss function, as depicted in \autoref{figure:line-search}.

The overall algorithm pipeline is shown in the following Algorithm \ref{alg:1}

\begin{algorithm}[!ht]
\caption{Line-search Method-Relative Learning Rate}
\begin{algorithmic}[1] 
\label{alg:1}
    \STATE \textbf{Initialize:} variables $\theta_0$, max iteration $N$, $\alpha_A$, $\alpha_I$, $\alpha_{max}$, $\alpha_{min}$, $\epsilon$, $c_1$
    \FOR{$i \gets 1$ to $N$}
        \STATE Calculate the update direction as $\mathbf{p}^t_A$ and $\mathbf{p}^t_I$
        \STATE Calculate the partial derivative of $\phi(\theta^{t-1}_A)$ about $\theta^{t-1}_A$ as direction denotes as $\nabla_{\theta^{t-1}_A} \phi$
        \STATE Calculate the partial derivative of $\phi(\theta^{t-1}_I)$ about $\theta^{t-1}_I$ as direction denotes as $\nabla_{\theta^{t-1}_I} \phi$
        \STATE $k\gets-\nabla_{\theta_A^{t-1}} \phi(\theta_A^{t-1}, \theta_I^{t-1})^\top \mathbf{p}^{t-1}_A+\epsilon$
        \STATE $b\gets \phi(\theta_A^{t-1}, \theta_I^{t-1}) -\alpha_I \nabla_{\theta^{t-1}_I} \phi^{\top}\mathbf{p}^t_I$
        \IF{$a\ge0, b\leq0$}
        \STATE $\alpha_A\gets\alpha_{min} $
        \STATE $\mathbf{Armijo}$($c_1, \nabla_{\theta^{t-1}_A} \phi^{\top}\mathbf{p}^t_A,\nabla_{\theta^{t-1}_I} \phi^{\top}\mathbf{p}^t_I$)
        \ELSIF{$a\ge0, b\geq0$ OR $a\le0, b\leq0$}
        \STATE $\alpha_A\gets Clip\left[|\frac{-b}{k}|, \alpha_{min}, \alpha_{max}\right] $
        \STATE $\mathbf{Armijo}$($c_1, \nabla_{\theta^{t-1}_A} \phi^{\top}\mathbf{p}^t_A,\nabla_{\theta^{t-1}_I} \phi^{\top}\mathbf{p}^t_I$)
        \ELSE
        \STATE $\alpha_A\gets\alpha_{min} $
        \STATE $\mathbf{Armijo}$($c_1, \nabla_{\theta^{t-1}_A} \phi^{\top}\mathbf{p}^t_A,\nabla_{\theta^{t-1}_I} \phi^{\top}\mathbf{p}^t_I$)
        \ENDIF
        \STATE $\theta^t_A\gets\theta^{t-1}_A+\alpha_A \mathbf{p}^t_A$
        \STATE $\theta^t_I\gets\theta^{t-1}_I+\alpha_I \mathbf{p}^t_I$
    \ENDFOR
\end{algorithmic}
\end{algorithm}

To ensure sufficient decrease, we also apply a similar derivation based on the Armijo condition, which is commonly used in line-search algorithms to guarantee sufficient decrease. The Armijo condition is typically expressed as follows:
\[
\phi(\theta^{t} + \alpha \mathbf{p}^{t}) \leq \phi(\theta^t) + c_1 \alpha \mathbf{p}^{t\top} \nabla \phi(\theta^{t}),
\]
Where $c_1$ typically takes the value $1 \times 10^{-3}$. By assuming the current step is $t$ (which is equal to the previous derivation's \( t-1 \), but we denote it as \( t \) for simplicity), this can be written as:
\begingroup
\scriptsize
\[
\phi(\theta^{t}_A +\alpha_A \mathbf{p}^t_A, \theta^{t}_I+\alpha_I\mathbf{p}^t_I)\]
\[\leq \phi(\theta^{t}_A, \theta^{t}_I) + c_1 \alpha_A \nabla_{\theta^{t}_A} \phi^{\top}\mathbf{p}^t_A + c_1\alpha_I \nabla_{\theta^{t}_I} \phi^{\top}\mathbf{p}^t_I
\]
\endgroup
Therefore, using a similar Taylor expansion on the loss function with respect to the parameters and algorithm is shown in Algorithm.\ref{alg:armijo}:
\begingroup
\scriptsize
\[
 \phi(\theta_A^{t}, \theta_I^{t}) - \nabla_{\theta_A^t} \phi(\theta_A^{t}, \theta_I^{t})^\top (\alpha_{A}\mathbf{p}^t_A) - \nabla_{\theta_I^t} \phi(\theta_A^{t}, \theta_I^{t})^\top (\alpha_{I}\mathbf{p}^t_I)
 \]
 \[
 \leq \phi(\theta^{t}_A, \theta^{t}_I) + c_1 \alpha_A p_k^\top \nabla_{\theta^{t}_A} \phi^{\top}\mathbf{p}^t_A + c_1\alpha_I \nabla_{\theta^{t}_I} \phi^{\top}\mathbf{p}^t_I
\]
\[\Rightarrow
(c_1-1)\alpha_I\nabla_{\theta^{t}_A}\phi^{\top}\mathbf{p}^t_A \geq (1-c_1)\nabla_{\theta^{t}_I} \phi^{\top}\mathbf{p}^t_I
\]
\endgroup
\begin{algorithm}[!ht]
    \caption{Armijo Condition Check (refered as Armijo)}
\begin{algorithmic}
    \STATE \textbf{Initialize:} variables $c_1$,$\nabla_{\theta^{t-1}_A} \phi^{\top}\mathbf{p}^t_A$, $\nabla_{\theta^{t}_I} \phi^{\top}\mathbf{p}^t_I$, current step t
    \IF{$(c_1-1)\nabla_{\theta^{t}_A} \phi^{\top}\mathbf{p}^t_A > (1-c_1)\nabla_{\theta^{t}_I} \phi^{\top}\mathbf{p}^t_I$}
        \IF{$\nabla_{\theta^{t}_I} \phi^{\top}\mathbf{p}^t_I\times\nabla_{\theta^{t-1}_A} \phi^{\top}\mathbf{p}^t_A>0$}
        \STATE \textbf{return} $\alpha_A\gets\frac{\nabla_{\theta^{t}_I} \phi^{\top}\mathbf{p}^t_I}{\nabla_{\theta^{t}_A} \phi^{\top}\mathbf{p}^t_A}$
        \ELSE
        \STATE \textbf{return} $\alpha_A\gets\alpha_{min}$
        \ENDIF 
    \ELSE        
    \STATE \textbf{return} $\mathbf{None}$
    \ENDIF
\end{algorithmic}
\label{alg:armijo}
\end{algorithm}

\clearpage
\subsection{Visualization of the Final Output of Filters}
\label{sec:visual_output}
In this section, we further present the results of the filtered Positional Encoding embedding as shown in \autoref{fig:sup1} and \autoref{fig:sup2}. Compared to Random Fourier Features, which involve more complex combinations of frequency components, Positional Encoding displays more regular frequency patterns, making it better suited for visualization. These visualizations demonstrate that in low-frequency regions, the high-frequency embeddings are effectively suppressed by the filter, in line with our expectations of the adaptive linear filter's behavior. Additionally, for low-frequency embeddings, the filter can also emphasize high-frequency components, enabling more fine-grained outputs.
\begin{figure}[!h]
    \centering
    \includegraphics[width=0.475\textwidth]{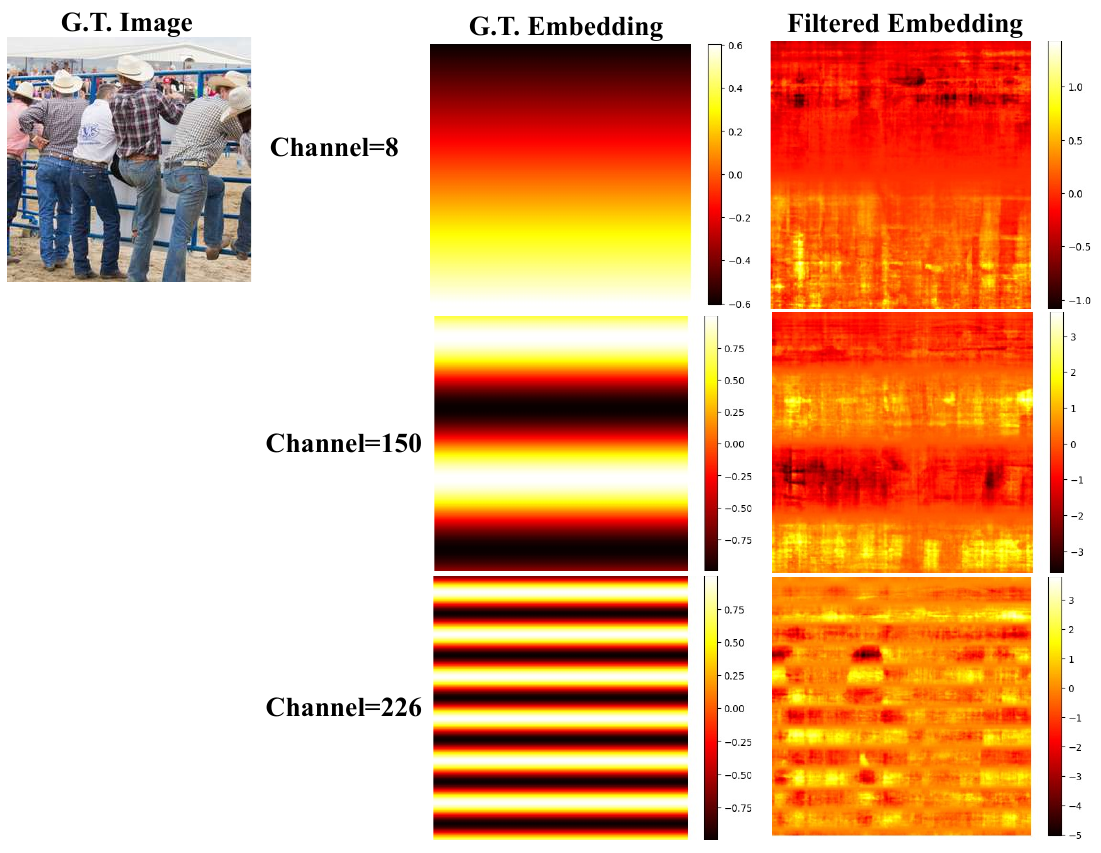}
    \caption{Visualization of the filtered embedding for image 804 in the DIV2K validation split.}
    \label{fig:sup1}
\end{figure}
\begin{figure}[!h]
    \centering
    \includegraphics[width=0.475\textwidth]{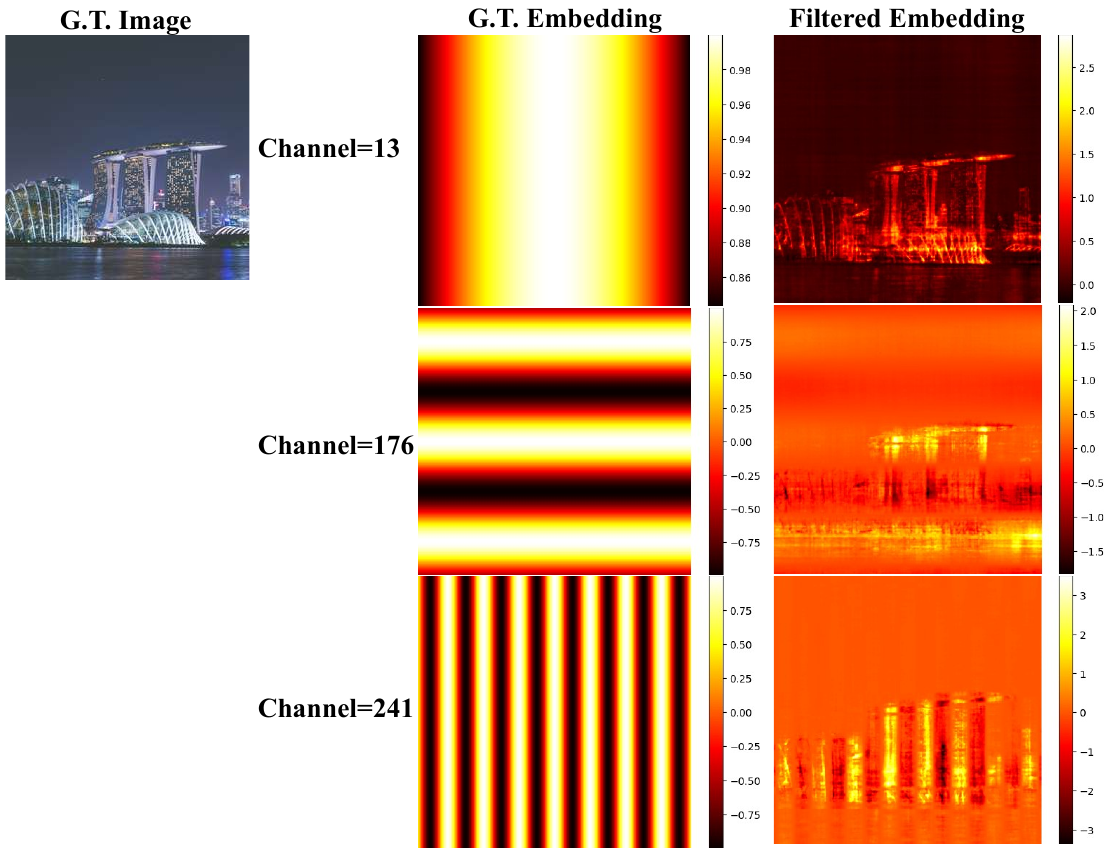}
    \caption{Visualization of the filtered embedding for image 814 in the DIV2K validation split.}
     \label{fig:sup2}
\end{figure}


\end{document}